\documentclass[conference]{IEEEtran}
\usepackage{times}
\usepackage{capt-of}
\usepackage[numbers]{natbib}
\usepackage{amsmath}
\usepackage{wrapfig}
\usepackage{algorithm}
\usepackage{algorithmic}
\usepackage{graphicx}
\usepackage{fontawesome}
\usepackage{multicol}
\usepackage[table]{xcolor}
\usepackage{csquotes}

\usepackage{booktabs}
\usepackage{multirow}
\usepackage{amssymb}
\usepackage{tcolorbox}
\usepackage[pagebackref=true, breaklinks=true, colorlinks,
            citecolor=citecolor, linkcolor=linkcolor, bookmarks=false]{hyperref}
\definecolor{citecolor}{HTML}{7AA6B8}
\definecolor{linkcolor}{HTML}{5A8805}
\usepackage{xcolor}
\definecolor{harpcolor}{HTML}{BE5D3D}
\newcommand{\ENAP}{\textcolor{harpcolor}{\textbf{ENAP}}}
\usepackage[capitalize]{cleveref}
\newtheorem{definition}{Definition}[section]
\newtheorem{proposition}{Proposition}[section] 
\definecolor{phasered}{HTML}{BE5D3D}
\definecolor{commentblue}{HTML}{678AAD}
\newcommand{\AlgPhase}[1]{%
    \STATE \textbf{\textcolor{phasered}{#1}}
}
\definecolor{colorAC}{RGB}{204,0,0}   
\definecolor{colorm67F}{RGB}{0,102,204}  
\definecolor{colorpCqu}{RGB}{0,153,76}  
\definecolor{colorMLDa}{RGB}{153,0,153} 
\definecolor{colorTh4H}{RGB}{204,102,0}

\begin{document}

\title{Emergent Neural Automaton Policies: Learning Symbolic Structure from Visuomotor Trajectories}

\author{Yiyuan Pan$^{*,1}$, Xusheng Luo$^{*,1}$, Hanjiang Hu$^{1}$, Peiqi Yu$^{1}$, Changliu Liu$^{1}$\\
$^{*}$Equal contribution\\
Robotics Institute, Carnegie Mellon University, Pittsburgh, PA 15213, USA\\
\{yiyuanp, xushengl, hanjianh, peiqiy, cliu6\}@andrew.cmu.edu \\
$\quad$\\
}

\makeatletter
\let\@oldmaketitle\@maketitle
    \renewcommand{\@maketitle}{\@oldmaketitle
    \centering
    \includegraphics[width=1.0\textwidth]{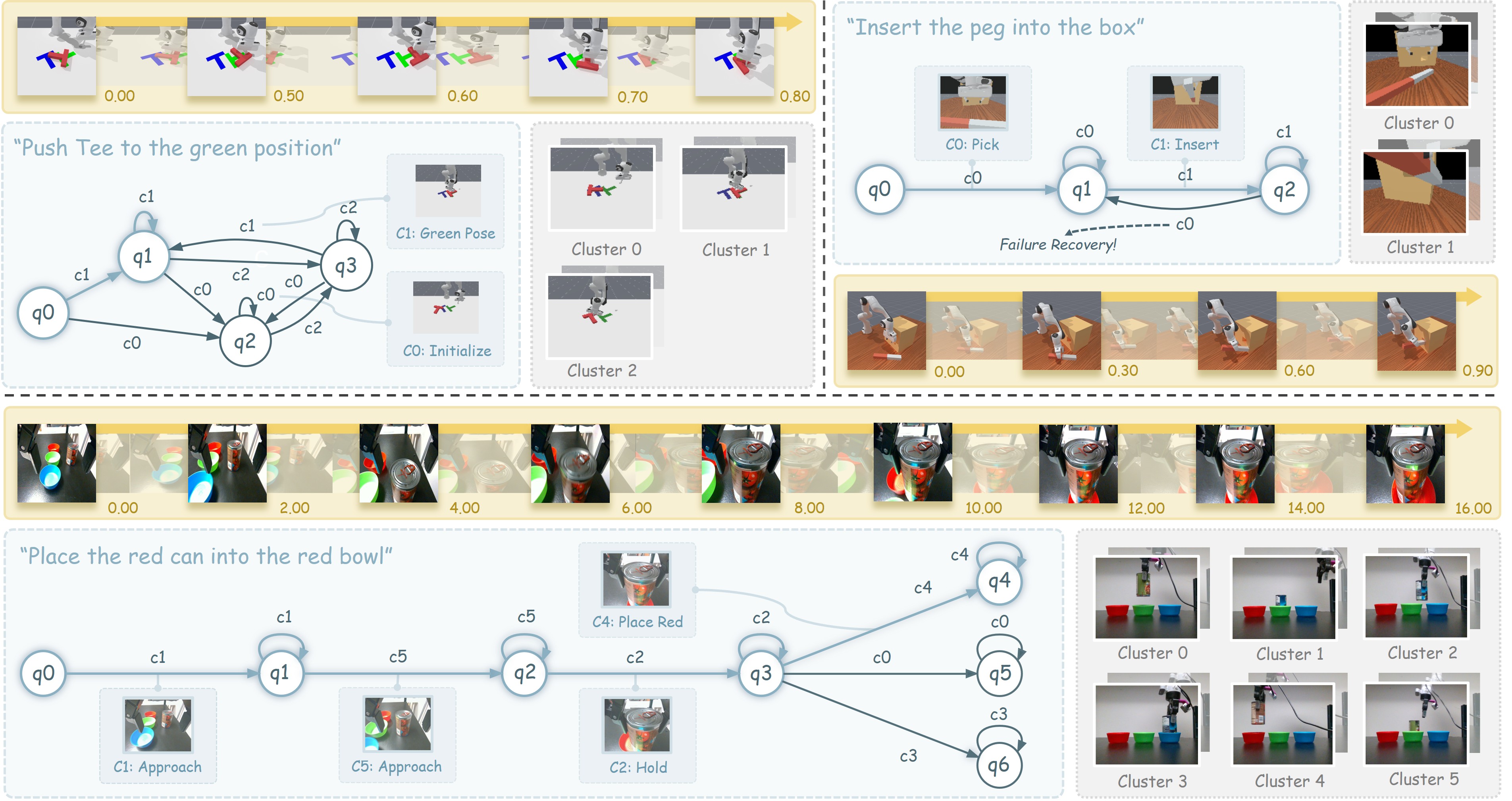}
    \captionof{figure}{
    \textbf{Unsupervised Discovery of Task Structures.} \ENAP\ successfully recovers task-specific logic, including cyclic dependencies (\textbf{Top-Left}), logical branching for multi-modal goals (\textbf{Bottom}), and crucially, \textit{autonomous failure recovery} (\textbf{Top-Right}, dashed line), where the automaton learns to transition back to a previous state ($q_2 \to q_1$) to retry insertion upon error. Semantic meanings (e.g., `\textit{Hold}') are generated by GPT based on cluster representatives. 
    \textbf{Webpage:} \url{https://intelligent-control-lab.github.io/ENAP-project-webpage/}
    }
    \vspace{-2pt} 
    \label{fig:firstpage}
    \setcounter{figure}{1}
  }
\makeatother
    
\maketitle

\begin{abstract}
Scaling robot learning to long-horizon tasks remains a formidable challenge. While end-to-end policies often lack the structural priors needed for effective long-term reasoning, traditional neuro-symbolic methods rely heavily on hand-crafted symbolic priors. We introduce \ENAP\ (\textcolor{harpcolor}{E}mergent \textcolor{harpcolor}{N}eural \textcolor{harpcolor}{A}utomaton \textcolor{harpcolor}{P}olicy), a neuro-symbolic imitation learning framework in which a bi-level policy adaptively emerges from visuomotor demonstrations. Specifically, we employ adaptive clustering and an extension of the $L^*$ algorithm to learn the task structure from demonstrations as a Mealy state machine, which serves as an interpretable high-level planner capturing latent task modes. This discrete structure then guides a low-level reactive residual network to learn precise continuous control via behavior cloning (BC). By explicitly modeling discrete transitions and continuous residuals, \ENAP\ achieves sample efficiency and interpretability without task-specific labels. Experiments on complex manipulation and long-horizon tasks show that \ENAP\ outperforms state-of-the-art end-to-end vision-language-action (VLA) policies by up to 27\% in low-data regimes. \ENAP\ can also distill task structure from the latent space of a pre-trained VLA, surpassing it on CALVIN benchmark with 39\% fewer parameters (\cref{fig:firstpage}).

\end{abstract}

\IEEEpeerreviewmaketitle



\section{Introduction}
Developing long-horizon execution capabilities is a critical step toward generalist robotic agents \cite{sermanet2024robovqa}. Such tasks inherently demand a synergy between high-level discrete planning and low-level continuous control, a challenge traditionally formulated within the realm of Task and Motion Planning (TAMP) \cite{guo2023recent,zhao2024survey}. While recent end-to-end learning approaches, such as Vision-Language-Action (VLA) models, have demonstrated efficacy in short-horizon tasks, scaling these methods to multi-stage and long-horizon problems remains challenging \cite{ma2024survey,kim2024openvla, pan2025seeing}. 

This limitation stems from a fundamental misalignment with the cognitive structure of intelligent decision-making. In psychology, decision-making processes in both humans and ideal robotic agents are naturally hierarchical: \textit{System 2} governs symbolic reasoning and long-term planning, while \textit{System 1} handles intuitive reactivity and sensorimotor control \cite{badcock2019hierarchically, macenski2022robot}. This hierarchical architecture not only facilitates learning by introducing structural priors but also improves interpretability compared to unstructured baselines. In contrast, end-to-end learning attempts to approximate the non-convex optimization landscape of long-horizon tasks using a single, monolithic network, resulting in a black-box policy that is semantically opaque and highly data-intensive \cite{kim2024openvla, pan2025planning,pan2025wonder}. Conversely, while existing neuro-symbolic TAMP frameworks successfully leverage bi-level planning for both performance and interpretability enhancement, they often rely on rigid, hand-crafted symbolic priors (e.g., temporal logic predicates) \cite{silver2022learning, silver2024neuro}, failing to adaptively reveal the intrinsic structural manifold from data.

Therefore, this work seeks to answer a pivotal question:

\begin{tcolorbox}[top=1pt, bottom=1pt, left=0pt, right=0pt]
\begin{center}
\textit{How can the logical structure of System 2 naturally emerge from the perceptual flow of System 1?}
\end{center}
\end{tcolorbox}

Hybrid system theory offers a profound insight: complex dynamics can be decomposed into sequences of continuous local behaviors governed by discrete modes \cite{lunze2009handbook, poli2021neural}. Inspired by this, we argue that one appropriate policy representation for long-horizon tasks is neither monolithic network policies nor rigid symbolic policies, but a ``Neural Automaton Policies". In this representation, a state machine encapsulates the global task phases (System 2), while continuous network functions model the local control within each phase (System 1). More importantly, this hybrid structure acts as a policy representation that is both kinematically feasible for robots and semantically interpretable for humans \cite{molnar2020interpretable}.

To this end, we introduce  \ENAP\ (\textcolor{harpcolor}{\textbf{E}}mergent \textcolor{harpcolor}{\textbf{N}}eural \textcolor{harpcolor}{\textbf{A}}utomaton \textcolor{harpcolor}{\textbf{P}}olicy), a framework that unsupervisedly reconstructs a task-level state machine planner from demonstrations (\cref{fig:method_overview}). \ENAP\ seamlessly integrates this discrete topological planner with a downstream residual compensation network within a unified framework. This bi-level architecture unifies discrete reasoning (where humans excel) and continuous control (where robots excel), aiming for both high interpretability and strong compositionality. Specifically, we first employ adaptive clustering and an extended $L^*$ algorithm to reconstruct the state machine from trajectory sets. Then, we combine this automaton-based planner with a residual network, learning the full policy via behavior cloning. Empirical results demonstrate that \ENAP\ achieves at least an 8\% absolute improvement over VLA models while using significantly 39\% fewer parameters and offering superior interpretability.

In summary, our contributions are threefold:
\begin{itemize}
    \item We propose \ENAP, which learns a structured task abstraction and low-level reactive network in an emergent manner, offering a viable path for evolving from end-to-end reasoning to hierarchical cognitive architectures.
    \item We design a label-free neuro-symbolic framework that eliminates the reliance on expert priors. We incentivize structure discovery through an adaptive symbolic abstraction, improving both interpretability and compositionality.
    \item We demonstrate that \ENAP\ consistently outperforms VLA models by at least 8\% in low-data settings, and provide a rigorous Partially Observable Markov Decision Process-based theoretical justification for the framework.
\end{itemize}

\section{Related Works}
\subsection{Neuro-Symbolic Learning for Robotics}
Neuro-symbolic approaches in robotics aim to coordinate the generalization capabilities of symbolic reasoning with the perceptual flexibility of neural networks, particularly within the domain of Task and Motion Planning (TAMP). While foundational frameworks like PDDLStream \cite{garrett2020pddlstream} successfully integrated discrete search with continuous motion planning, they relied on manually specified symbolic models. To overcome this, recent research has pivoted toward learning abstractions directly from data. In the realm of action abstraction, approaches have evolved from discovering temporal options via skill chaining \cite{konidaris2009skill} to learning deep parameterized skills \cite{pertsch2021accelerating,da2012learning} that enable continuous control within symbolic plans. Complementary work in state abstraction seeks to ground continuous sensorimotor data into discrete symbols \cite{konidaris2018skills}, utilizing techniques ranging from deep generative models \cite{asai2018classical} to neuro-symbolic concept learners \cite{mao2019neuro} to construct latent states compatible with symbolic planners. Building upon these learned abstractions, model learning focuses on inferring the transition dynamics that govern symbolic interactions. This area has progressed from learning probabilistic rules in stochastic domains \cite{pasula2007learning} to employing Graph Neural Networks (GNNs) for modeling relational dynamics \cite{chitnis2022learning} and hybrid automata for stochastic transitions \cite{poli2021neural}. Finally, to enable execution across varying task configurations, Generalized planning methods train neural policies to approximate symbolic plans \cite{groshev2018learning}, often utilizing GNNs to reason about object importance in large-scale environments \cite{silver2021planning} or to select appropriate planners \cite{ma2020online}.

While prior neuro-symbolic methods successfully integrate learning and planning, they typically require pre-defined symbolic predicates. In contrast, \ENAP\ introduces a label-free, self-supervised framework that autonomously discovers interpretable automata directly from raw visuomotor data, bridging the gap between low-level sensorimotor control and high-level symbolic reasoning without manual engineering.

\subsection{Symbolic Structure Discovery}
Discovering symbolic structures from continuous data is pivotal for enhancing interpretability and planning efficiency in robotics. A broad spectrum of methods has been proposed to infer such structures, ranging from learning PDDL operators and predicates for bilevel planning \cite{li2025bilevel,lorang2025few} to discovering symbolic graphs and rules. Among these, ``learning from demonstration" is a significant research topic \cite{krishnan2017transition,44shah2023supervised,wang2022temporal}, ranging from learning grounded symbolic representations for planning \cite{11konidaris2018skills} to inferring temporal logic formulas via maximum entropy \cite{22vazquez2018learning}, or joint learning of logical structure and atomic propositions \cite{33chou2020explaining}. 

While these approaches successfully capture high-level task logic, Finite State Machines (FSMs) offer a unique advantage: they provide a compact, topological representation of sequential dynamics that supports formal verification and reactive control. Specifically, Mealy machines have been shown to offer richer structural representations than Moore machines, making them particularly effective for POMDP controllers \cite{amato2010finite}.

Driven by these advantages, a prominent line of research focuses on automata extraction from data. Building on classical approaches like the $L^*$ algorithm and the Myhill-Nerode theorem \cite{dhayalkar2025neural}, recent works have successfully extracted Deterministic Finite Automata (DFA) from RNNs to verify formal properties \cite{danesh2006re,koul2018learning} and from Transformers to reveal their operational mechanisms \cite{zhang2024automata}. To address stochasticity, extensions such as Weighted Finite Automata (WFA) extraction have been proposed for natural language processing \cite{wei2022extracting}. In the robotics domain, latent automata models like LATMOS \cite{zhan2025latmos} and serialized state machines \cite{mu2025look} have been developed to factorize long-horizon tasks into executable subgoals. Similarly, learning probabilistic automata (PDFA) has proven effective for capturing temporal dependencies and expert preferences directly from demonstrations \cite{baert2025learning}.

While existing methods excel in specific niches, they typically require either discretized inputs or predefined knowledge like cluster counts. \ENAP\ bridges these gaps by combining self-supervised continuous embedding with adaptive automata inference, enabling the emergent discovery of probabilistic, task-relevant structures without manual specification.

\section{Background}
\subsection{Problem Setup}
We formulate the robotic task as a Partially Observable Markov Decision Process (POMDP) \cite{cassandra1998survey}, defined by the tuple $\mathcal{P} = (S, \mathcal{O}, A, \mathcal{T}, \Omega, R, \gamma)$. Here, $S$ represents the state space of the environment, $\mathcal{O}$ denotes the high-dimensional observation space (e.g., RGB-D images, proprioception), and $A \subseteq \mathbb{R}^d$ is the continuous action space. The transition dynamics are given by $\mathcal{T}(s'|s,a)$, and the observation model by $\Omega(o|s,a)$.



We operate in the Learning from Demonstration (LfD) setting. We assume access to a dataset of expert trajectories $\mathcal{D} = \{\tau_i\}_{i=1}^N$, collected from expert demonstrations or rollouts of a policy. Each trajectory is a sequence $\tau_i = (o_0, a_0, \dots, o_{T_i}, a_{T_i})$, where the observation at time $t$ is explicitly defined as a tuple $o_t := (I_t, p_t)$, comprising high-dimensional visual input $I_t$ (e.g., RGB images) and proprioceptive state $p_t$ (e.g., joint pose). The objective is to learn a structured policy that mimics the expert's behavior.

\subsection{Preliminaries}
\subsubsection{\textbf{Mealy Machine}}
Finite State Machines (FSMs) are mathematical models of computation used to design logic for sequential systems. In the context of robotic planning and control \cite{kress2018synthesis,luo2025simultaneous}, FSMs' nodes represent discrete operational modes and edges define transitions triggered by perceptual events. The two primary machines are \textit{Moore machines} \cite{giantamidis2021learning}, where outputs are determined solely by the current state, and \textit{Mealy machines}, where outputs depend on both the current state and the current input \cite{shahbaz2009inferring}. 

In this work, we adopt the \textit{Probabilistic Mealy Machine} (PMM) as our structural backbone. While Moore machines are intuitive, Mealy machines offer greater reactivity and compactness \cite{li2006equivalence}. Moreover, this structure mirrors the fundamental dynamics of a POMDP, where the agent's action and state evolution are jointly determined by the current latent state and observation \cite{amato2010finite}. Formally, a PMM is defined as follows:

\vspace{5pt}

\begin{definition}[Probabilistic Mealy Machine]
\label{def:PMM}
The Probabilistic Mealy Machine is $\mathcal{M} = (Q, \Sigma, \Gamma, \delta, \lambda, q_0)$, where:
\begin{itemize}
    \item $Q$ is a finite set of states.
    \item $\Sigma$ is a finite input alphabet.
    \item $\Gamma$ is a finite output alphabet.
    \item $\delta(q' | q, \sigma)$, the probabilistic transition function representing the probability of moving to state $q'$ given current state $q$ and input $\sigma$;
    \item $\lambda(\gamma | q, \sigma)$, the probabilistic output function representing the probability of emitting output $\gamma\in \Gamma$ given $q$ and $\sigma$;
    \item $q_0 \in Q$ is the initial state.
\end{itemize}
\end{definition}

\vspace{5pt}

To bridge this automata-based definition with the POMDP formulation ($\mathcal{P}$) in our paper, we establish the following conceptual mapping: alphabet $\Sigma$ corresponds to a discretization of the continuous observation space $\mathcal{O}$ (see \cref{fig:firstpage} for example); nodes $Q$ serve as discrete phases inferred from the POMDP history; output alphabet $\Gamma$ represents the action space $\mathcal{A}$. Crucially, the transition function $\delta(q'|q, \sigma)$ models the high-level task progression (transitioning between phases), while the output function $\lambda(\gamma|q, \sigma)$ acts as a state-conditioned policy.


\begin{table}[t]
\centering
\caption{Comparison of Symbolic Structure Discovery Approaches.}
\vspace{-5pt}
\label{tab:comparison_methods}
\resizebox{\linewidth}{!}{%
\begin{tabular}{lcccc}
\hline
\textbf{Method}\rule{0pt}{7pt} & \textbf{Continuous} & \textbf{Probabilistic} & \textbf{Label-Free} \\
\hline
\multicolumn{4}{l}{\cellcolor{gray!10}\textit{- PDDL / Rule-Based}\rule{0pt}{7pt}} \\
Bilevel Learning [6] & $\checkmark$ &   &   \\
Neuro-Symbolic IL [10] & $\checkmark$ &   &   \\
\multicolumn{4}{l}{\cellcolor{gray!10}\textit{- Automata Extraction (NLP/Theory)}\rule{0pt}{7pt}} \\
RNN Extraction [2,3] &   &   & $\checkmark$ \\
WFA Extraction [5] &   & $\checkmark$ & $\checkmark$ \\
Transformer Extr. [4] &   &   & $\checkmark$ \\
\multicolumn{4}{l}{\cellcolor{gray!10}\textit{- Robotic Automata Learning}\rule{0pt}{7pt}} \\
PDFA Learning [7] & $\checkmark$ & $\checkmark$ &   \\
LATMOS [8] & $\checkmark$ &   &   \\
Serialized FSM [9] & $\checkmark$ &   &   \\
\hline
{\ENAP\ (Ours)}\rule{0pt}{8pt} & \textbf{$\checkmark$} & \textbf{$\checkmark$} & \textbf{$\checkmark$} \\
\hline
\end{tabular}%
}
\end{table}

\subsubsection{\textbf{Automata learning and L$^*$ algorithm}}

\begin{figure*}[t]
    \centering
    \includegraphics[width=\linewidth]{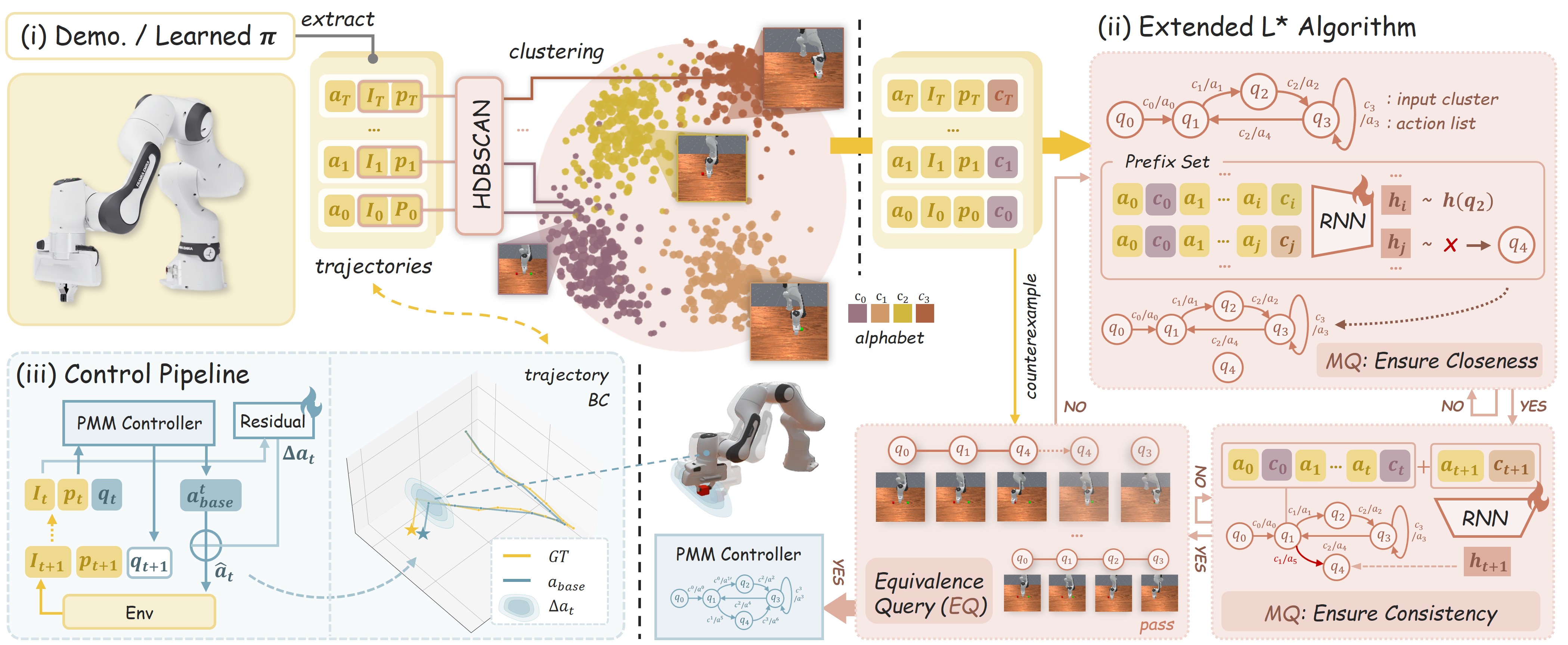} 
    \caption{\textbf{Overview of the \ENAP\ Framework.} 
    Our approach unifies structure discovery and hierarchical control through a three-stage pipeline:
    \textbf{(i) Symbol Abstraction:} Continuous trajectories from demonstrations are encoded and clustered via HDBSCAN \cite{mcinnes2017hdbscan} to discover a discrete alphabet $\Sigma$, mapping sensorimotor streams to symbolic sequences.
    \textbf{(ii) Structure Extraction via Extended L$^*$:} We introduce an extended $L^*$ algorithm that iteratively constructs a PMM by querying the dataset. The process maintains an observation table and enforces closedness (\textit{expanding nodes}) and consistency (\textit{refining edge}) constraints using an RNN-based history encoder.
    \textbf{(iii) Bi-level Control Pipeline:} The learned PMM serves as a high-level planner, providing a coarse action prior $a_{\text{base}}^t$. This is combined with a learned residual term $\Delta a_t\sim\pi_{\psi}(\cdot|q_t,o_t)$ to synthesize the final precise action $\hat{a}_t = a_{\text{base}}^t + \Delta a_t$.}
    \label{fig:method_overview}
\end{figure*}

\textit{Automata learning} aims to infer an FSM that explains a given set of observed sequences. Traditional algorithm is Angluin's $L^*$ \textit{Algorithm} \cite{angluin1987learning}. Its core principle rests on the Myhill-Nerode theorem  \cite{fellows1989analogue}, which provides a rigorous definition of ``node" (state) in FSMs: two history sequences are equivalent (i.e., belong to the same node) if and only if they yield identical future behaviors for all possible subsequent inputs.  $L^*$ leverages this to actively query the system to cluster history sequences into discrete states. 

Specifically, $L^*$ algorithm maintains an \textit{Observation Table} $(\mathcal{U}, \mathcal{E}, T)$ to infer the automaton. Let $\Sigma^*:=\{ \sigma_1 \cdots \sigma_k \mid k \in \mathbb{N}, \sigma_i \in \Sigma \}$ denote the set of finite input sequences, and $\mathcal{U} \cdot \Sigma = \{ u \cdot \sigma \mid u \in \mathcal{U}, \sigma \in \Sigma \}$ denote concatenation operation. Then, $\mathcal{U} \subset \Sigma^*$ is a set of \textit{Prefixes}, which can be viewed as candidate history sequences leading to distinct states; $\mathcal{E} \subset  \Sigma^*$ is a set of \textit{Suffixes}, representing future test sequences used to distinguish between two prefix sequences $u_i,u_j\in\mathcal{U}$. The table entries $T: \mathcal{U}\cdot \mathcal{E} \rightarrow \Gamma^*$ record the queried system's output (``what happens if we execute history $u \in \mathcal{U}$ followed by test $e \in \mathcal{E}$") for a prefix followed by a suffix. The learning proceeds via interactions with a Minimally Adequate Teacher (MAT) \cite{clark2010distributional}: 
\begin{enumerate}
    \item \textbf{Membership Query (MQ)}: the learner asks for the system output of a specific sequence $u \cdot e$, filling the table entries in $T$.
    \item \textbf{Equivalence Query (EQ)}: The learner proposes a hypothesis automaton based on the current table. The teacher either confirms it or provides a \textit{counter-example} trajectory that reveals an undiscovered behavior.
\end{enumerate}

With queries, the algorithm can ensure two key properties of the table: \textit{Closedness} and \textit{Consistency}. \textit{Closedness} ensures that for every discovered state, its transition to any next state is also covered by the known set $\mathcal{U}$. \textit{Consistency} ensures that if two histories appear identical under current tests $\mathcal{E}$, they must remain identical after any one-step extension. If either property is violated, $L^*$ expands $\mathcal{U}$ or $\mathcal{E}$ accordingly. We provide an illustrative example of $L^*$ Algorithm in Appendix~\ref{app:L}.

However, applying classical $L^*$ directly to robotic learning presents a significant gap: standard $L^*$ assumes discrete symbolic inputs and access to an omniscient oracle. Moreover, it is originally designed for deterministic automata. Therefore, we extend the $L^*$ algorithm to enable learning of PMM.

\section{Methodology}
\noindent\textbf{Method Overview. }
Given the demonstration dataset $\mathcal{D}$, our objective is to learn a structured policy $\pi_{\text{ENAP}}$ comprising a high-level PMM controller $\mathcal{M}$ and a low-level and compensatory residual network. \textit{Firstly}, we perform symbol abstraction to discretize the continuous observation space $\mathcal{O}$ into an alphabet $\Sigma$, and compress sequences of symbols into task phases (nodes) using a Recurrent Neural Network (RNN). \textit{Second}, we introduce an extension of the $L^*$ algorithm for continuous trajectory data. \textit{Finally}, we train a residual policy network $\psi$ via behavior cloning to ensure precise control, which takes the coarse action prior $\lambda(\cdot|q, \sigma)$ from the learned PMM and predicts a continuous action. (\cref{fig:method_overview}).

In what follows, we try to answer two questions: (i) \textit{How to unsupervisedly construct an interpretable state machine from noisy demonstrations? (\cref{sec:grounding,sec:extended_lstar})} and (ii) \textit{How to leverage such learned structure to guide the learning of hierarchical control policy? (\cref{sec:PMM_controller,sec:residual_learning})}

\subsection{Adaptive Symbol Abstraction}
\label{sec:grounding}
In the \ENAP\ framework, we follow \cref{def:PMM} and instantiate the PMM. Specifically, the node set $Q$ represents the discrete latent task modes (e.g., \textit{Reach}, \textit{Grasp}, \textit{Align}), while the edges encode the transitions between these nodes. 

Given the continuous trajectory data, the first step of \ENAP\ is to map observations $o_t := (I_t, p_t)$ to a semantic feature vector $z_t = \phi_{\theta}(o_t) \in \mathbb{R}^z$ via an encoder. Following this, we employ HDBSCAN~\cite{mcinnes2017hdbscan}, an adaptive clustering algorithm, to unsupervisedly discover the alphabet $\Sigma$ within the feature space. Each cluster leads to one symbol. This step augments the trajectory data at each time step from $(o_t, a_t)$ to a tuple $(o_t, a_t, c_t)$, where $c_t = {\rm HDBSCAN}(z_t) \in \Sigma$ is the assigned cluster label. The choice of encoder depends on whether the trajectory data are obtained from demonstrations or from rollouts of the learned policy.
\begin{enumerate}
    \item \textbf{Scenario 1 (Policy Enhancement)}: The trajectory data are obtained from policy rollouts. In this setting, our objective is to extract a PMM to enhance interpretability or to construct a structured policy with improved performance. The encoder learned within the policy is directly reused.
    \item \textbf{Scenario 2 (Policy Discovery)}: The trajectory data are obtained from demonstrations. Our objective is to learn a policy from scratch using the provided demonstrations (e.g., teleoperation data). In this setting, we fine-tune a raw visual encoder (e.g., DINO \cite{ougab2023dinov2}) using our proposed algorithm (see \cref{alg:iterative_learning}).
\end{enumerate}

Given the clustered symbols, standard $L^*$ algorithms rely on exact matching of sequences to distinguish nodes. In robotics, a sequence refers to a trajectory of historical actions and symbols $(a_{0:t}, c_{0:t})$. To handle the variable-length nature of robotic trajectories where different temporal spans may map to the same task phase, we train an RNN to map the history $(a_{0:t}, c_{0:t})$ to a fixed-size embedding $h_t \in \mathcal{H}$.
Intuitively, $h_t$ serves as a continuous surrogate for discrete state nodes, summarizing variable-length past experiences into a consistent Markovian representation of the current task phase.
(the number of $h$ equals the total number of time steps across all trajectories). Specifically, we attach auxiliary prediction heads to the RNN during training and discard them afterward in order to optimize a multi-objective loss:
\begin{equation}
    \mathcal{L}_{\text{RNN}} := \mathcal{L}_{\text{pred}} + \lambda\mathcal{L}_{\text{contrast}} = (\mathcal{L}_{\text{act}} + \mathcal{L}_{\text{state}}) + \lambda \mathcal{L}_{\text{contrast}}
\end{equation}
where the loss terms $\mathcal{L}_{\text{act}}$ (MSE) and $\mathcal{L}_{\text{state}}$ (cross-entropy) are employed to predict the subsequent action $a_{t+1}$ and symbolic state $c_{t+1}$, respectively. Phase-aware contrastive loss $\mathcal{L}_{\text{contrast}}$ to separate latent representations across different phases: it minimizes the cosine distance between $h_t$ and $h_{t+1}$ during periods of constant cluster assignment (self-loops, $c_t = c_{t+1}$) while maximizing it during phase transitions ($c_t \neq c_{t+1}$). 

This architectural choice of RNN is theoretically motivated: the trained RNN induces clustering of continuous hidden states around saturated centers, in which each component is driven toward $\pm 1$ under tanh activations. As a result, cosine similarity becomes a reliable metric for distinguishing automaton states (see Appendix~\ref{app:rnn_theory} for the formal proposition).




\subsection{Structure Extraction via Extended $L^*$ Algorithm}
\label{sec:extended_lstar}
To bridge classical $L^*$ with continuous robot learning, we translate its discrete logical properties into spatial geometric constraints. Intuitively, \textit{closedness} (traditionally requiring a new string to match a known state) now means any newly observed history embedding must fall within the $\tau_{\text{sim}}$-neighborhood of an existing centroid. Similarly, \textit{consistency} (requiring identical states to yield identical futures) is relaxed to ensure that trajectories within the same current node, given the same input $c$, consistently transition into the same destination region.

Leveraging the abstracted alphabet $\Sigma$ and the phase embeddings $\mathcal{H}$, we propose an extension of the $L^*$ algorithm. Due to the absence of an oracle, our algorithm queries the static demonstration database $\mathcal{D}$ to infer system behavior. We formulate this as an iterative process, driven by two redefined query mechanisms: a \textit{Self-Supervised Membership Query} (MQ) that is used to construct a hypothetical state machine, and a \textit{Non-deterministic Equivalence Query} (EQ) that validates the hypothesis (\cref{alg:harp_discovery}).

\begin{algorithm}[tb]
   \caption{Extended $L^*$ Algorithm}
   \label{alg:harp_discovery}
\begin{algorithmic}[1]
   \STATE {\bfseries Input:} Demonstration $\mathcal{D}$, RNN $h_\kappa$, Thresholds $\tau_{\text{sim}}, \epsilon_{\text{err}}$
   \STATE {\bfseries Output:} Probabilistic Mealy Machine $\mathcal{M}_{\text{final}}$
   \STATE Construct Database $\mathcal{H}$.
   \STATE Initialize prefix set $\mathcal{U} \leftarrow \{ h_0^1 \}$.
   
   \WHILE{Hypothesis doesn't pass \textit{EQ}}
      \AlgPhase{\textit{Phase 1}: Self-Supervised MQ}
      \REPEAT
         \STATE $\text{\textit{Closed}} \leftarrow \text{True}$
         \STATE \textit{\textcolor{commentblue}{// Check if all transitions land in known states}}
         \STATE Construct transition destinations $\{h_{\text{next}} \}$ reachable from $\mathcal{U}$ via database $\mathcal{H}$.
         \IF{$\max_{u' \in \mathcal{U}} \text{sim}(h_{\text{next}}, u') < \tau_{\text{sim}}$}
             \STATE $\mathcal{U} \leftarrow \mathcal{U} \cup \{ h_{\text{next}} \}$; ${\rm \textit{Closed}} \leftarrow \text{False}$
         \ENDIF
      \UNTIL{\text{\textit{Closed}} is \text{True}}
      \STATE Construct hypothesis $\mathcal{M}$ from $\mathcal{H}$ supported by $\mathcal{U}$.

      \vspace{4pt}

      \AlgPhase{\textit{Phase 2}: Non-deterministic EQ}
      \FORALL{test trajectory $\tau \in \mathcal{D}_{\text{test}}$}
         \STATE Track all paths in $\mathcal{M}$ consistent with $\tau$'s symbols.
         \STATE \textit{\textcolor{commentblue}{// Validity requires topology and action match}}
         \IF{no valid path survives at step $t$}
             \STATE $\mathcal{U} \leftarrow \mathcal{U} \cup \{ h(\tau{[:j]}), \forall j\leq t \}$; \textbf{break} \textit{\textcolor{commentblue}{// Add Counterexample}}
         \ENDIF
      \ENDFOR
      
   \ENDWHILE

   \vspace{4pt}
   
   \AlgPhase{\textit{Phase 3}: Stable Phase Pruning}
   \STATE $\mathcal{M}_{\text{final}} \leftarrow \textbf{Prune}(\mathcal{M})$ 

   \vspace{4pt}
   
   \STATE \textbf{return} $\mathcal{M}_{\text{final}}$
\end{algorithmic}
\end{algorithm}
\begin{algorithm}[tb]
   \caption{Iterative Residual Learning}
   \label{alg:iterative_learning}
\begin{algorithmic}[1]
   \STATE {\bfseries Input:} Dataset $\mathcal{D}$, Initial Encoder $\phi_{\theta^{(0)}}$, Iterations $K$
   \STATE {\bfseries Output:} Trained Policy $\pi_{\text{ENAP}} = (\mathcal{M}, \pi_{\psi}, \phi_{\theta})$
   
   \FOR{$k = 0$ \TO $K-1$}
      \AlgPhase{E-Step: Structure Extraction}
      \STATE Extract features $\phi_{\theta^{(k)}}(o_t)$ for all $\tau \in \mathcal{D}$.
      \STATE Discretize features to $\Sigma^{(k)}$ via HDBSCAN.
      \STATE Train RNN on $(\Sigma^{(k)}, \mathcal{D})$ to get $h^k_{\kappa}$.
      \STATE $\mathcal{M}^{(k)} \leftarrow \textbf{Extended $L^*$}(\mathcal{D}, h^k_{\kappa}, \dots)$ \textit{\textcolor{commentblue}{// See \cref{alg:harp_discovery}}}
      
      \vspace{4pt}

      \AlgPhase{M-Step: Policy Optimization}
      \STATE Freeze PMM $\mathcal{M}^{(k)}$, initialize Residual Net $\pi_{\psi^{(k)}}$.
      \REPEAT
         \STATE Sample batch $(o_t, a_t, q_t)$ from $\mathcal{D}$.
         \STATE Compute loss $\mathcal{J} \leftarrow$ \cref{loss}.
         \STATE Update $\theta^{(k)}, \psi^{(k)}$ via gradient descent.
      \UNTIL{Convergence}
   \ENDFOR
   
   \STATE \textbf{return} $(\mathcal{M}^{(K)}, \pi_{\psi^{(K)}}, \phi_{\theta^{(K)}})$
\end{algorithmic}
\end{algorithm}

\subsubsection{\textbf{Self-Supervised MQ} (\textcolor{harpcolor}{\textbf{\textit{Phase 1}}})}

In our framework, the prefix set $\mathcal{U}$ serves as the collection of representative history embeddings, where each $u \in \mathcal{U}$ is a centroid summarizing a cluster of similar trajectory embeddings $h_t$ that correspond to a unique latent task node $q$. Two embeddings $h^i_t$ and $h^j_{t'}$ from $i$-th and $j$-th trajectories at different time steps may correspond to the same $u \in \mathcal{U}$, that is, task phase. The core challenge is to identify the PMM from the non-Oracle dataset $\mathcal{D}$. To this end, we define the Generalized MQ function $\text{MQ}_{\mathcal{D}}(u)$, which retrieves the immediate action $a_t$ and next-step embeddings $h_{t+1}$ from all trajectory segments in $\mathcal{D}$ that are $\tau_{sim}$-cosine similar to the representative history $u$, thereby obtaining transitions for the candidate state,
\begin{equation}
\begin{aligned}
\text{MQ}_{\mathcal{D}}(u)
&= \{ (a_t^i, h^i_{t+1}) \mid \\
&\quad \exists \ t \text{ s.t. }
\cos\!\big(h(\tau_i[:t]), u\big) \geq \tau_{\text{sim}} \}.
\end{aligned}
\end{equation}

We initialize $\mathcal{U}$ using the initial state of the first trajectory at $t = 0$, that is, $\mathcal{U} =\{h_0^1\}$ and use $\text{MQ}_{\mathcal{D}}(u)$ to expand $\mathcal{U}$. The growth of $\mathcal{U}$ continues until the hypothesis PMM satisfies both \textit{closedness} (every reachable next-step embedding is similar to some $u\in\mathcal{U}$), at which point the construction early-stops. Specifically, for every prefix $u$ in current $\mathcal{U}$, we query $\text{MQ}_{\mathcal{D}}(u)$ to obtain all reachable next-step embeddings $\{h_{\text{next}}\}$. If a reachable $h_{\text{next}}$ satisfies $\max_{u' \in \mathcal{U}} \cos(h_{\text{next}}, u') < \tau_{\text{sim}}$, which means the $h_{\text{next}}$ is dissimilar to every recorded history, we add the $h_{\text{next}}$ into $\mathcal{U}$ as a new centroid, expanding the automaton's nodes until closedness is satisfied. Once $\mathcal{U}$ is closed, each $u \in \mathcal{U}$ corresponds to a node $q$. To this end, we derive the finite set of states $Q$ of the PMM. 

Given constructed \textit{closed} $Q$, the next step is to construct the set of outgoing transitions that satisfy \textit{consistency}. Specifically, in the sequel, we discuss how we build edge from node $q$ under input symbol $c\in\Sigma$, and associate them with two critical pieces of information: (1) An action prior distribution $\lambda(\cdot|q, c)$, represented practically as the empirical mean $a_{\rm base}$ of all continuous action vectors observed transitioning along this edge; and (2) A next-input set ${\rm NIS}(q) \subset \Sigma$, which records all the valid input symbols accepted by node $q$. 


Specifically, for a discrete state $q$ (represented by a centroid $u$) and an input symbol $c$, the generalized query scans the demonstration dataset to retrieve all continuous next-step embeddings $\{h_{t+1}\}$ from time steps $t$ where the history is similar to $u$ and the input is $c$. We then map each continuous $h_{\text{next}}$ to a discrete destination state $q'$ by finding its nearest representative $u' \in \mathcal{U}$. Next, we estimate the discrete transition probability $\delta(q'|q, c)$ based on the normalized frequency of reaching each $q'$. Finally, the continuous action prior $\lambda(\cdot|q, c)$ is modeled as the empirical distribution (the mean value mentioned earlier) of all continuous actions $a_t$ executed during these specific $q \xrightarrow{c} q'$ transitions. In our cases, the \textit{consistency} property is satisfied by default, as the different instances $h$ within the same node will transit to the same next node under the same input symbol.



In classical $L^*$, the algorithm constructs a complete transition table by actively querying an oracle for the outcome of every state prefix $u$ extended by every possible input symbol $\sigma \in \Sigma$. However, in our offline imitation learning setting, we lack an interactive oracle, and the static demonstration dataset naturally suffers from transition sparsity, meaning many $(u, \sigma)$ combinations are physically unobserved. To overcome this limitation, our Generalized MQ redefines the query process: instead of exhaustively testing all hypothetical extensions, it acts as a data-mining operation, retrieving only the empirically observed next-step embeddings $h_{\text{next}}$ that actually follow the current state in the dataset.

\subsubsection{\textbf{Non-deterministic EQ} (\textcolor{harpcolor}{\textbf{\textit{Phase 2}}})}
For robotic tasks, we relax the deterministic verification to a non-deterministic path search. Given a test trajectory $\tau = (o_0, a_0, \dots)$, we track all possible transition sequences in the PMM induced by its cluster-index sequence. If there exists at least one valid transition sequence $\tau^{\rm PMM} = (c_0, a^0_{\rm base}, \dots)$ (i.e., at every time step the test action $a_t$ lies close to the action mean $a^t_{\rm base}$ of the edge under a given tolerance threshold), then the EQ is passed. Otherwise, the shortest prefix of the trajectory $\tau$ that causes the failure is returned as a counterexample and added to $\mathcal{U}$, forcing the learner to split states or introduce new transitions to resolve the ambiguity.

\subsubsection{\textbf{Stable Phase Pruning} (\textcolor{harpcolor}{\textbf{\textit{Phase 3}}})}
To recover a compact structure from the initial temporal expansion, we merge a destination node $q'$ into its source node $q$ when the transition from $q$ to $q'$ is triggered by $c$ and $q$ contains a self-loop on $c$, which mirrors a stable phase. 

\begin{figure}[t]
    \centering
    \includegraphics[width=0.7\linewidth]{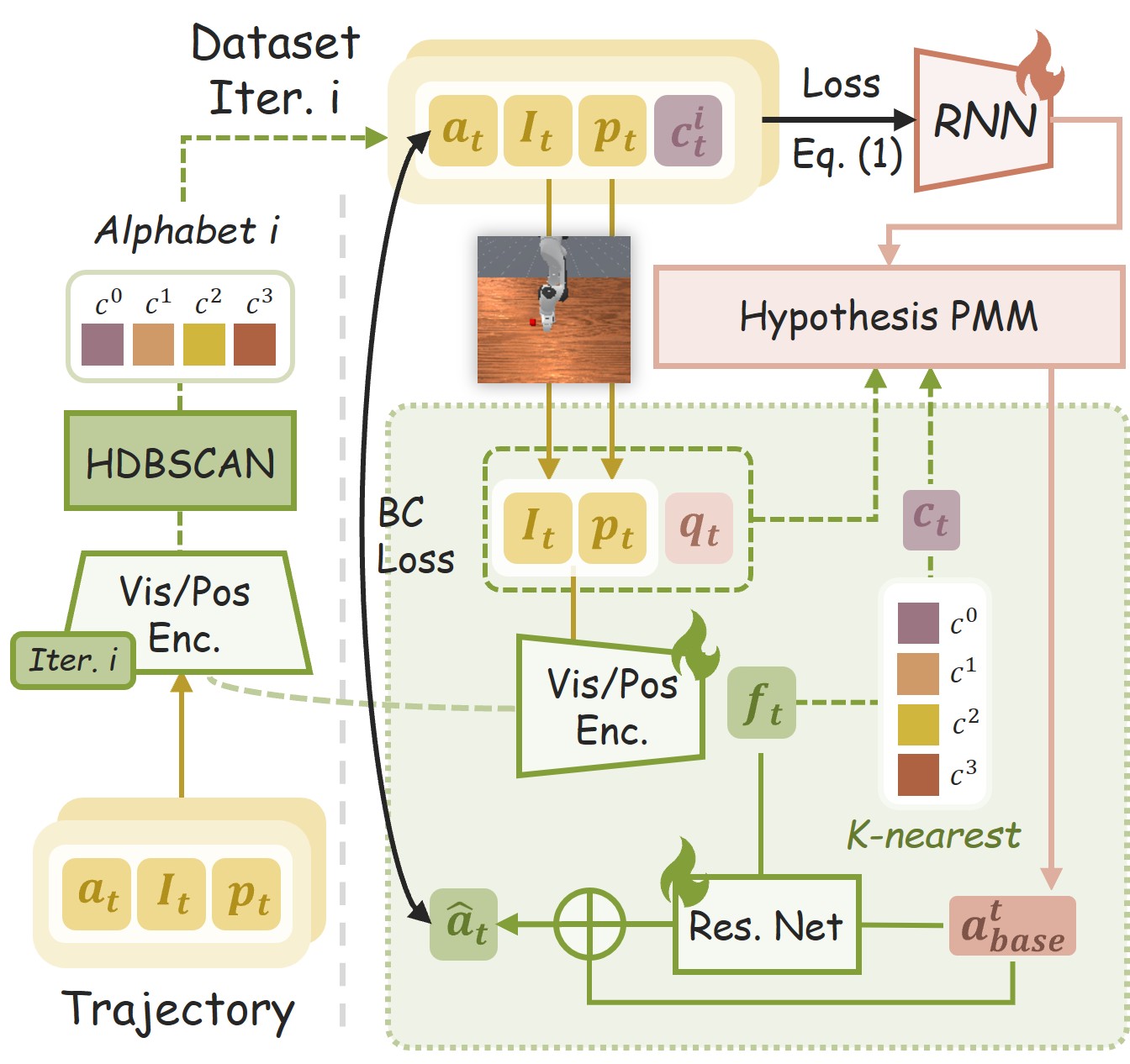}
    \caption{\textbf{Iterative Structure-Policy Co-Evolution.} 
    The training process alternates between augmenting dataset via clustering with learned encoder (\textbf{Left}) and policy learning under the given dataset (\textbf{Right}). A residual network refines the action prior from PMM and generate updated encoder for subsequent iterations.}
    \label{fig:training_loop}
\end{figure}

\subsection{PMM as a Reactive Controller}
\label{sec:PMM_controller}
After obtaining a PMM from data, the inference logic of this controller operates through a deterministic predict-act-update cycle. At time step $t$, with current state $q_t$ and grounded input $c_t$, the PMM first generates a coarse action prior $a_{\text{base}}^{t}$. Then, given the global guidance $a_{\text{base}}^{t}$, the fine-grained control is realized by the residual policy $\pi_{\psi}$ and the encoder $\phi_\theta$.
\begin{subequations}
\begin{align}
    a_{\text{base}}^{t}  &= \mathbb{E}[\lambda(\cdot|q_t, c_t)], \\
    c_t  & = \operatorname{HDBSCAN}\!\left(\phi_\theta(o_t)\right), \\
    \hat{a}_{\rm t}
&=
\underbrace{a_{\text{base}}^t}_{\text{Fixed}}
+ \pi_{\psi}\!\left(
q_t,\,
\phi_\theta(o_t),\,
{a}_{\text{base}}^t
\right). \label{eq:ac_pred}
\end{align}
\end{subequations}

Upon execution, the controller updates its internal state to $q_{t+1}$. In our formulation, the transition may remain ambiguous due to the probabilistic $\lambda$ function. To alleviate, we exploit the observed next symbol $c_{t+1}$. Specifically, after executing the action, the controller observes $c_{t+1}$ and selects the successor state whose outgoing edge signature contains this observation. Formally, the next state is determined as follows, where $\mathbb{I}_{(\cdot)}$ prioritizes transitions whose next-input includes $c_{t+1}$, and $P(q' \mid q_t, c_t)$ is used to select the most probable transition when multiple nodes $q'$ are reachable from $q$ under input $c_{t+1}$:
\begin{equation}
    q_{t+1} = \underset{q' \in \delta(\cdot|q_t,c_t)}{\arg\max}
    \Big[
    \mathbb{I}_{c_{t+1} \in {\rm NIS}(q')} 
    + \epsilon \cdot P(q' \mid q_t, c_t)
    \Big],
\end{equation}

\begin{figure*}[t]
    \centering
    \includegraphics[width=\linewidth]{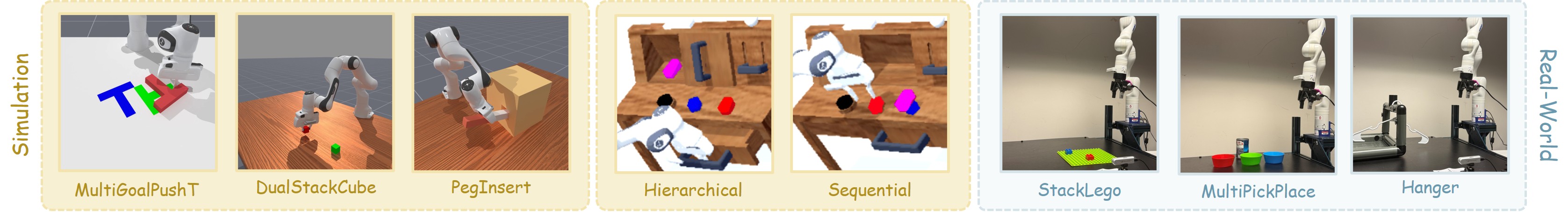}
    \caption{\textbf{Experiments in Both Simulation and Real-World.} (\textbf{Left}) Complex manipulation; (\textbf{Middle}) TAMP; (\textbf{Right}) Real-world scenarios.}
    \label{fig:exp_setup}
\end{figure*}

\subsection{Iterative Residual Learning and Refinement}
\label{sec:residual_learning}
To identify an effective feature space that maximizes performance, we propose an iterative Expectation–Maximization (EM) training paradigm (see \cref{alg:iterative_learning} and \cref{fig:training_loop}) to jointly train the encoder $\phi(\theta)$ and the residual network $\pi(\psi)$.
Intuitively, because discrete clustering operations block gradient backpropagation, the encoder's features and the resulting state machine structure cannot be optimized end-to-end.
Thus, instead of a static one-shot pipeline, we use the EM-style algorithm to process alternates between refining the structural hypothesis (E-step) and optimizing the control policy (M-step), driven by the following joint objective:
\begin{equation}
   \mathcal{J}(\theta, \psi)
=
\mathbb{E}_{\tau \sim \mathcal{D}} \Big[
\big\| a_t - \hat{a}_{\rm t} \big\|^2
+ \lambda_{\text{reg}} \, \mathcal{L}_{\text{center}}
\Big]. \label{loss}
\end{equation}
where $\mathcal{L}_{\text{center}}$ is a regularization term pulling features towards cluster centers to maintain separability, formally defined as $\mathcal{L}_{\text{center}} = \frac{1}{2} \| f_t - \mu_{c_t} \|^2$, where $f_t$ is the encoded feature and $\mu_{c_t}$ is the centroid of its assigned cluster $c_t$.

\subsubsection{\textbf{E-Step (Structure Extraction)}} 
Given the current encoder parameters $\theta^{(k)}$, we freeze the feature space and perform the structure discovery pipeline. This yields updated cluster assignments $\Sigma^{(k+1)}$ and a refined PMM $\mathcal{M}^{(k+1)}$. Crucially, this step re-segments the task based on the feature space learned from the previous control iteration.

\subsubsection{\textbf{M-Step (Policy Optimization)}} 
Fixing $\mathcal{M}^{(k+1)}$, we optimize the residual policy $\pi_\psi$ and fine-tune the encoder $\phi_\theta$ via behavior cloning. The gradient descent update is:
\begin{equation}
    \theta^{(k+1)}, \psi^{(k+1)} \leftarrow \text{Optimizer} \left( \nabla_{\theta, \psi} \mathcal{J}(\theta^{(k)}, \psi^{(k)}) \right)
\end{equation}

This EM-style training procedure ensures that the encoder learns a task-relevant manifold where ``states" are both structurally consistent (for PMM) and control-effective (for residuals), avoiding trivial solutions such as collapsing to a single node or overfitting to every time step.

\section{Experiments}
In this section, we present extensive experimental results to address the following questions:
\begin{itemize}
    \item \textbf{Q1}: Can \ENAP\ achieve improved control efficacy with extracted System 2?
    \item \textbf{Q2}: Can \ENAP\ extract semantically meaningful structures from unlabeled demonstrations?
    \item \textbf{Q3}: Can \ENAP\ exhibit the unique advantages of state machines over other symbolic structures?
\end{itemize}

\subsection{Experiment Setup}
We evaluate across three domains, as demonstrated in Fig.~\ref{fig:exp_setup}; detailed specifications are in the Appendix~\ref{app:experiments}. In the sequel, we denote our method as \ENAP\ $(\cdot)$, where $\cdot$ specifies the policy from which the state machine is extracted. We use \textcolor{harpcolor}{\textbf{ENAP$^*$}} (DINO) to indicate we train DINO from scratch using \cref{alg:iterative_learning}. All tasks are evaluated with a success rate metric.

\subsubsection{\textbf{Complex Manipulation}} \texttt{PegInsert}: focus on precision insertion; \texttt{DualStackCube}: stack blocks (green, red) in variable orders; \texttt{MultiGoalPushT}: push a Tee to the either of diverse goals. In the settings, we compare with classic behavior cloning models \cite{gmm,transformer} and recent large-scale / generative policies \cite{dp,kim2024openvla,pi0}. Moreover, to approximate an ideal extracted state machine, we obtain an \textit{Oracle} policy, which reliably captures the task structure, using reward-shaping training or motion planning. All models are trained on identical demonstration sets on Maniskill \cite{mu2021maniskill} simulator.
\subsubsection{\textbf{Long-Horizon TAMP}} CALVIN \cite{mees2022calvin} is a widely used benchmark for long-horizon task and motion planning, 
requiring the robot to complete multi-stage manipulation sequences. 
We evaluate on two tasks with up to five stages. \texttt{Sequential} (\texttt{Seq.}): rotate block $\to$ push block $\to$ move slider $\to$ toggle switch $\to$ light bulb, where skills follow a fixed linear order. \texttt{Hierarchical} (\texttt{Hier.}): open drawer $\to$ lift block $\to$ place in drawer $\to$ lift second block $\to$ stack, where later actions depend on intermediate subgoals (e.g., robot must open the drawer before placing the cube into it). In the settings, we evaluate against multiple VLA baselines \cite{hulc,lcd,mdt,flower}, with their best-reported checkpoints.
\subsubsection{\textbf{Real-World Manipulation}} \texttt{StackLego}: Precise Lego block assembly; \texttt{MultiPickPlace}: color-matched object sorting; \texttt{Hanger}: hang the hanger on a rack. In the settings, we benchmark against the fine-tuned $\pi_{0.5}$ model \cite{pi05}.

\begin{table}[t]
\centering
\caption{{Comparison on Complex Manipulation Tasks.}}
\label{tab:manipulation_results}
\vspace{-7pt}
\resizebox{\linewidth}{!}{%
\begin{tabular}{lccc}
\hline
\multirow{2}{*}{\textbf{Method}} & \textbf{Param} 
& \cellcolor[HTML]{F8F0D2} \texttt{DualStack} 
& \cellcolor[HTML]{F8F0D2} \texttt{Peg} \\
& (M) 
& \cellcolor[HTML]{F8F0D2} \texttt{Cube} (\%) 
& \cellcolor[HTML]{F8F0D2} \texttt{Insert} (\%) \\
\hline
\multicolumn{4}{l}{\cellcolor{gray!10}\textit{- Oracle}} \\
Oracle & $2.98$ & $98.3\,{\scriptstyle \pm 0.4}$ & $86.7\,{\scriptstyle \pm 0.8}$ \\
\multicolumn{4}{l}{\cellcolor{gray!10}\textit{- Traditional Imitation Learning}} \\
Transformer \cite{transformer} & $63.81$ & $38.7\,{\scriptstyle \pm 6.0}$ & $51.8\,{\scriptstyle \pm 5.5}$ \\
GMM \cite{gmm} & $46.11$ & $73.6\,{\scriptstyle \pm 2.3}$ & $53.1\,{\scriptstyle \pm 2.6}$ \\
\multicolumn{4}{l}{\cellcolor{gray!10}\textit{- Generative \& VLA Policies}} \\
Diffusion Policy \cite{dp} & $114.39$ & $41.2\,{\scriptstyle \pm 7.2}$ & $31.1\,{\scriptstyle \pm 6.8}$ \\
OpenVLA \cite{kim2024openvla} & $7652.10$ & $69.8\,{\scriptstyle \pm 2.0}$ & $42.3\,{\scriptstyle \pm 2.8}$ \\
$\pi_0$ \cite{pi0} & $3288.52$ & $73.4\,{\scriptstyle \pm 1.2}$ & $51.6\,{\scriptstyle \pm 1.4}$ \\
\multicolumn{4}{l}{\cellcolor{gray!10}\textit{- State-Machine-Based}} \\
\ENAP\ (Oracle) & $\mathbf{2.66}$ & $98.8\,{\scriptstyle \pm 0.3}$ & $85.6\,{\scriptstyle \pm 0.6}$ \\
\textcolor{harpcolor}{\textbf{ENAP$^*$}} (DINO) & 22.94 & $\mathbf{76.0\,{\scriptstyle \pm 2.0}}$ & $\mathbf{63.2\,{\scriptstyle \pm 2.4}}$ \\
\hline
\end{tabular}%
} 
\end{table}
\begin{figure}[t]
    \centering
    \includegraphics[width=\linewidth]{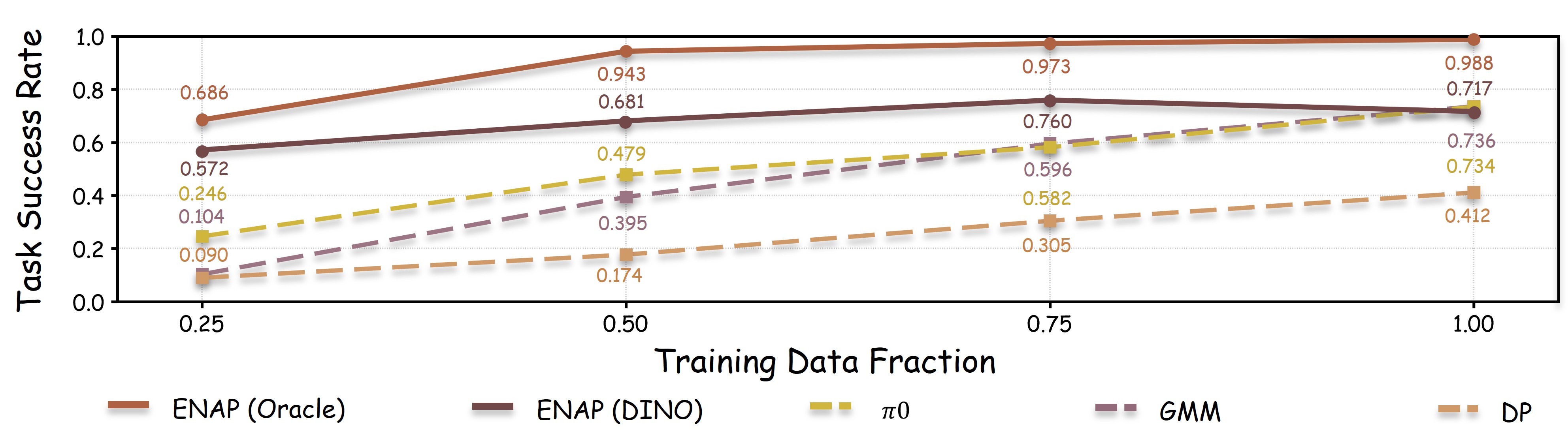}
    \caption{\textbf{Data-scaling comparison on} \texttt{DualStackCube}. \ENAP\ remains robust in low-data regimes, while other methods degrade significantly.}
    \label{fig:scaling}
\end{figure}

\begin{figure*}[t]
    \centering
    \includegraphics[width=\linewidth]{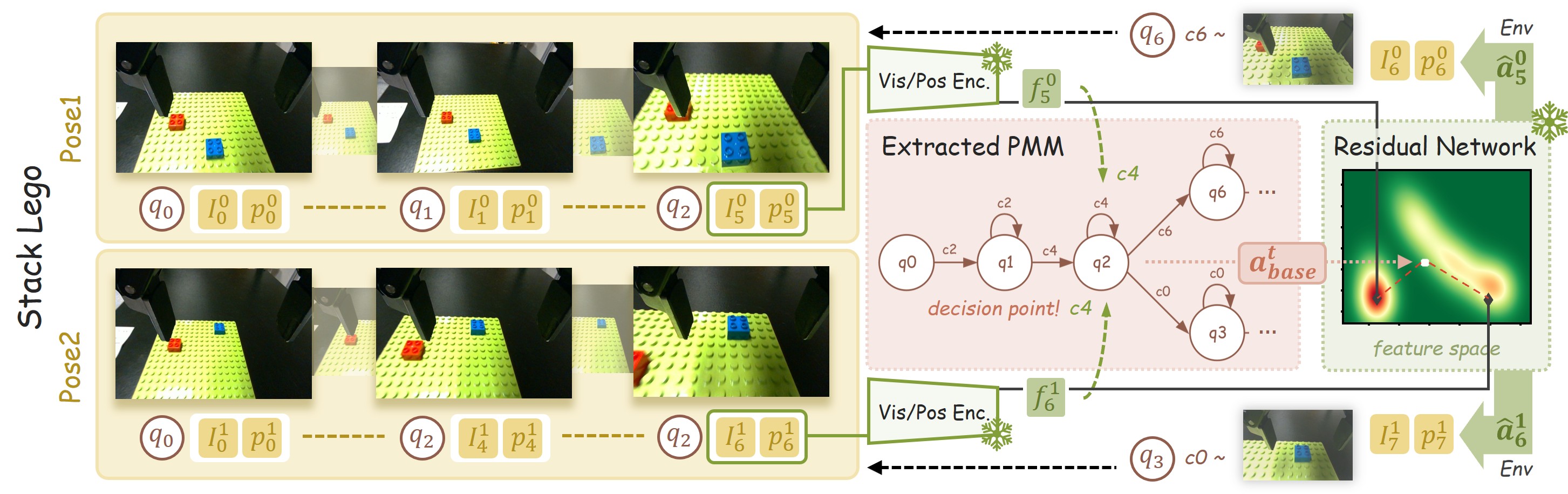}
    \caption{
    \textbf{Mechanism of Logical Branching at Decision Points.}
    We visualize how \ENAP\ handles multi-modal \texttt{StackLego} tasks. 
    At the critical decision node ($q_2$), the extracted PMM identifies a branching structure based on distinct future outcomes ($q_3$ vs.\ $q_6$). 
    By conditioning on the current observation, the residual policy guides the transition into the correct logical mode. 
    Here $I_t^i$ and $p_t^i$ denote the image observation and pose at timestep $t$ of trajectory $i$. 
    The predicted action $\hat{a}_t$ follows \cref{eq:ac_pred}. 
    The feature vector $f_t$ is obtained by encoding the observation $o_t = [I_t; p_t]$ via the encoder $\phi_\theta$, i.e., $f_t = \phi_\theta(o_t)$. 
    For the state machine involving more Lego start poses, please refer to \cref{app-sec:aa}.
    }
    \label{fig:qualitative}
\end{figure*}

\subsection{Control Efficacy \& Efficiency  (\textbf{Q1})}

The quantitative results (Table \ref{tab:manipulation_results}) demonstrate \ENAP's ability to efficiently distill expert behavior. When using Oracle's encoder (\ENAP\ (Oracle)), our method fully achieves comparable success rate while reducing the parameter count. This confirms that the extracted PMM serves as a compact, structured representation that captures the task's essence without the overhead of massive monolithic networks.

Crucially, even without access to a task-specific encoder, our framework remains effective (\cref{tab:manipulation_results}, \cref{tab:realworld_results}). By employing our proposed \cref{alg:iterative_learning} with the DINOv2 backbone, we achieve competitive success rates. This highlights \ENAP's potential as a general-purpose algorithm capable of extracting effective policies without reliance on expert priors.

To rule out the possibility that our advantage stems merely from model size suitability, we evaluate performance across varying data fractions on \texttt{DualStackCube}. As shown in \cref{fig:scaling}, \ENAP\ maintains high efficacy even in extremely low-data regimes (25\% data). This efficiency validates that the extracted PMM captures the true underlying task topology, acting as a strong inductive bias that generalizes beyond the training distribution.

\begin{table}[t]
\centering
\caption{{Real-World Evaluation}.}
\label{tab:realworld_results}
\resizebox{\linewidth}{!}{%
\begin{tabular}{lccccc}
\hline
\multirow{2}{*}{\textbf{Method}} & \textbf{Param} & \textbf{Speed} 
& \cellcolor[HTML]{D6E3E9} \texttt{Stack} 
& \cellcolor[HTML]{D6E3E9} \texttt{Pick} 
& \cellcolor[HTML]{D6E3E9} \texttt{Hanger} \\
& (M) & (ms) 
& \cellcolor[HTML]{D6E3E9} \texttt{Lego} 
& \cellcolor[HTML]{D6E3E9} \texttt{Place} 
& \cellcolor[HTML]{D6E3E9} \texttt{Task}  \\
\hline
$\pi_{0.5}$ \cite{pi05}\rule{0pt}{7pt} & $3403$ & $6841$ & $58.82$ & $76.47$ & $64.71$ \\
\textcolor{harpcolor}{\textbf{ENAP$^*$}} (DINO)& $23$ & $281$ & $\mathbf{88.24}$ & $\mathbf{94.12}$ & $\mathbf{94.12}$ \\
\hline
\end{tabular}%
}
\end{table}
\begin{table}[t]
\centering
\caption{{Comparison on Long-Horizon TAMP Tasks.}}
\label{tab:tamp_results}
\resizebox{\linewidth}{!}{%
\begin{tabular}{lcccc}
\hline
\multirow{2}{*}{\textbf{Method/Param (M)}}
& \multicolumn{2}{c}{\cellcolor[HTML]{F8F0D2} $\texttt{Seq.}$ (\%)} 
& \multicolumn{2}{c}{\cellcolor[HTML]{F8F0D2} $\texttt{Hier.}$ (\%)} \\
\cline{2-3}\cline{4-5}
& \cellcolor[HTML]{F8F0D2} $3/5$ 
& \cellcolor[HTML]{F8F0D2} $5/5$ 
& \cellcolor[HTML]{F8F0D2} $3/5$ 
& \cellcolor[HTML]{F8F0D2} $5/5$ \\
\hline


HULC ($47$) \cite{hulc} \rule{0pt}{8pt}
& $3.0\,{\scriptstyle \pm 0.3}$ 
& $3.0\,{\scriptstyle \pm 0.2}$ 
& $87.0\,{\scriptstyle \pm 0.6}$ 
& $2.0\,{\scriptstyle \pm 0.2}$ \\

LCD ($68$) \cite{lcd} 
& $11.0\,{\scriptstyle \pm 0.5}$ 
& $9.0\,{\scriptstyle \pm 0.4}$ 
& $57.2\,{\scriptstyle \pm 0.7}$ 
& $5.0\,{\scriptstyle \pm 0.3}$ \\

MDT ($440$) \cite{mdt} 
& $78.4\,{\scriptstyle \pm 0.9}$ 
& $53.7\,{\scriptstyle \pm 1.1}$ 
& $93.8\,{\scriptstyle \pm 0.8}$ 
& $21.9\,{\scriptstyle \pm 0.6}$ \\

FLOWER ($947$) \cite{flower} 
& $91.0\,{\scriptstyle \pm 0.6}$ 
& $90.6\,{\scriptstyle \pm 0.5}$ 
& $90.8\,{\scriptstyle \pm 0.7}$ 
& $15.9\,{\scriptstyle \pm 0.4}$ \\


\ENAP\ (FLOWER) ($571$) 
& $\mathbf{97.0\,{\scriptstyle \pm 0.4}}$ 
& $\mathbf{96.8\,{\scriptstyle \pm 0.3}}$ 
& $\mathbf{95.5\,{\scriptstyle \pm 0.5}}$ 
& $\mathbf{28.2\,{\scriptstyle \pm 0.6}}$ \\
\hline
\end{tabular}%
}
\end{table}

For sequential reasoning tasks (\cref{tab:tamp_results}; in the table, 3/5 and 5/5 denote completing the first 3 or 5 subtasks, respectively.), we evaluate \ENAP\ by distilling structure directly from the latent space of a pre-trained VLA. We use FLOWER \cite{flower}, an efficient flow-based vision–language–action policy that generates robot actions from multimodal observations and language goals. It demonstrates that our unsupervised structure discovery is applicable even to the token-level representations (i.e., the transformer token embeddings encoding visual and language observations), thus offering a general enhancement to the interpretability and performance of existing black-box policies.

\begin{figure}[htbp]
    \centering
    \includegraphics[width=\linewidth]{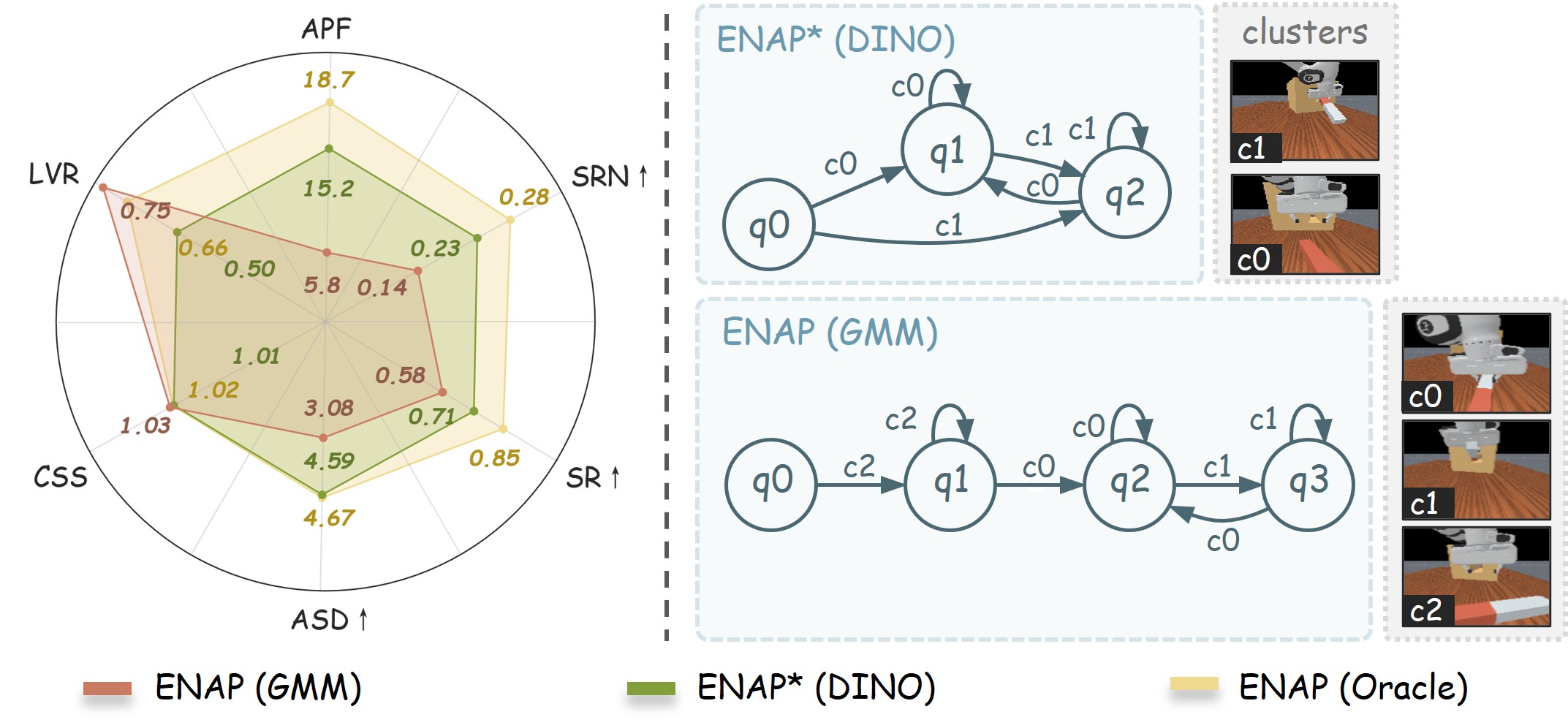}
    \caption{\textbf{Comparison of Learned Policies from} \texttt{PegInsert}. \textbf{(Left)} Radar chart of six structural metrics. \textbf{(Right)} Visualized PMMs for DINO and GMM variants. For the PMM from Oracle, see \cref{fig:firstpage}.}
    \label{fig:rader}
\end{figure}

\subsection{Emergent Structure Analysis (\textbf{Q2})}
For qualitative analysis, we refine the feature space by re-clustering with K-Means, using the cluster count inferred by HDBSCAN, to eliminate noise artifacts.

To validate the interpretability of our learned representation, we visualize the execution of a simplified real-world task: stacking Lego blocks with blue brick at two starting locations (\cref{fig:qualitative}). 
The extracted PMM autonomously discovers a decision point at node $q_2$ (the ``\textit{Grasp}" phase), where multiple valid task branches exist due to the various task configurations (e.g., grasping the brick from multiple starting location). 
During inference, the specific path selection is implicitly guided by the residual network: given the visual input, the network predicts a precise action that leads to a next-step observation consistent with only one of the valid transition signatures.

To move beyond qualitative inspection, we conduct a rigorous structural evaluation on the \texttt{PegInsert} task, comparing policies derived from Oracle, DINO, and GMM encoders (\cref{fig:rader}). We introduce six metrics: 
\begin{enumerate}
    \item \textbf{Success Rate (SR$\uparrow$)};
    \item \textbf{Success Rate weighted by Node Count (SRN$\uparrow$):} $SR / |Q|$, the efficiency of the state representation;
    \item \textbf{Action Prior Fidelity (APF):} The MSE between the PMM's coarse prior $a_{\text{base}}$ and the ground truth (extremely low values may indicate overfitting to specific trajectories);
    \item \textbf{Loop-Transition Ratio (LVR):} The ratio of self-loops to total edges;
    \item \textbf{Cluster Semantic Separability (CSS$\uparrow$):} The ratio of intra-cluster to inter-cluster visual similarity. For each cluster, we select the representative image closest to the cluster center and compute its CLIP embedding; CSS is then computed from the pairwise embedding similarities between cluster representatives; higher values indicate semantically distinct task phases;
    \item \textbf{Action Separation Distance (ASD$\uparrow$):} The average distance between mean actions of different edges. Higher values confirm that the automaton effectively partitions the control space into distinct primitives. 
\end{enumerate}
Among these metrics, \textbf{SR}, \textbf{SRN}, and \textbf{ASD} are performance-oriented measures where higher values indicate better policies, while \textbf{APF}, \textbf{LVR}, and \textbf{CSS} characterize structural properties of the learned automaton and serve primarily for qualitative structural analysis rather than direct optimization.

The radar chart reveals an insight: the best-performing model does not simply maximize or minimize all structural metrics. Firstly, \ENAP\ (Oracle)'s \textbf{LVR} and \textbf{CSS} strike a balance between phase stability and dynamic responsiveness, indicating an optimal task structure is a resilient abstraction that partitions the control space distinctly for the residual policy to handle. The high prior uncertainty (\textbf{APF}) suggests that the Oracle-derived PMM avoids overfitting to specific trajectory noise, learning a generalized ``coarse intent".

\subsection{State-Machine-Enabled Property  (\textbf{Q3})}
\subsubsection{\textbf{Recovery via Structural Loops}}
A unique advantage of the PMM controller is to perform autonomous failure recovery through structural loops. In the \texttt{StackLego} task (\cref{fig:recovery}), where precise alignment is challenging without force feedback, we observe that the agent frequently exploits the self-loop on the ``Align" node (\cref{fig:firstpage}). If the visual feedback does not match the ``Insert" signature, the controller repeatedly re-attempts the adjustment. This behavior emerges naturally from the learned topology, enhancing robustness against uncertainties.
\begin{figure}[htbp]
    \centering
    \includegraphics[width=\linewidth]{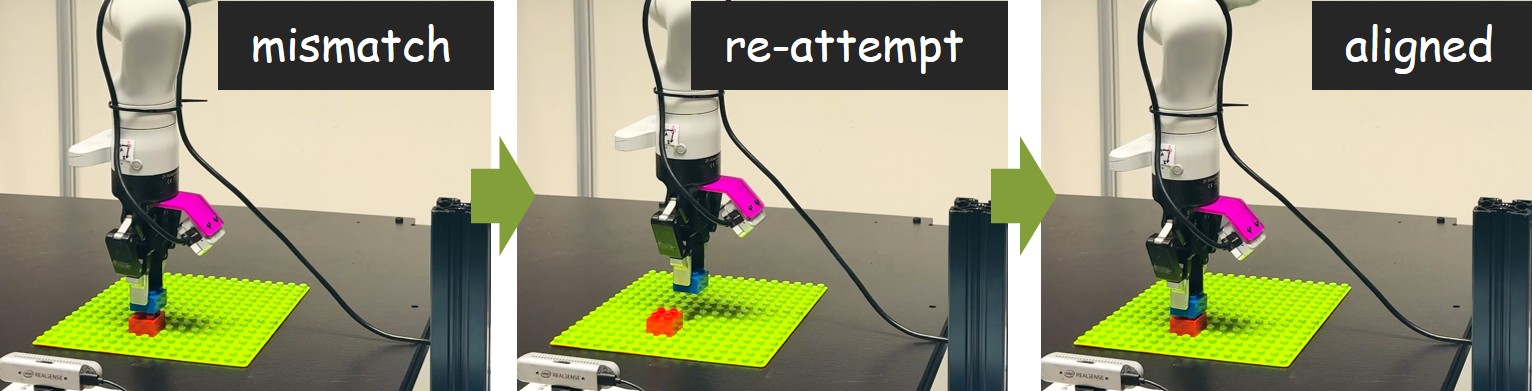}
    \caption{\textbf{Failure Recovery in} \texttt{StackLego}.}
    \label{fig:recovery}
\end{figure}

\begin{table}[t]
\centering
\caption{{MultiGoalPushT Performance Comparison.}}
\label{tab:multigoal_pusht}
\resizebox{\linewidth}{!}{%
\begin{tabular}{lccccc}
\hline
\cellcolor[HTML]{F8F0D2}\textbf{Metric} (\%) \rule{0pt}{7pt} 
& {DP} \cite{dp} 
& {GMM} \cite{gmm}
& $\pi_0$ \cite{pi0}
& {\ENAP\ (Oracle)} \\
\hline
\cellcolor[HTML]{F8F0D2}{Either}\rule{0pt}{7pt}     
& $0.41\,{\scriptstyle \pm 0.05}$
& $0.50\,{\scriptstyle \pm 0.05}$
& $0.25\,{\scriptstyle \pm 0.02}$
& $\mathbf{0.94\,{\scriptstyle \pm 0.03}}$ \\
\cellcolor[HTML]{F8F0D2}{Nearest}     
& $0.38\,{\scriptstyle \pm 0.07}$
& $0.38\,{\scriptstyle \pm 0.03}$
& $0.13\,{\scriptstyle \pm 0.04}$
& $\mathbf{0.84\,{\scriptstyle \pm 0.02}}$ \\
\hline
\end{tabular}%
}
\end{table}

\subsubsection{\textbf{Generalization via Policy Mixtures}}
We further investigate \ENAP's capability to distill a generalist state machine from mixtures of heterogeneous policies. {The fundamental advantage of policy mixture is skill transfer: by combining scarce demonstrations from similar tasks into a unified, general dataset, the agent can leverage shared experiences (e.g., universal pushing mechanics) to improve overall proficiency across all sub-tasks.} To evaluate this, we employ the \texttt{MultiGoalPushT} task, where the agent must push a T-block to either a blue or green target. Instead of training a monolithic policy, which is notoriously difficult on multi-modal distributions due to divergent optimal trajectories averaging out conflicting gradients, we generate our training dataset by combining trajectories from two distinct specialist policies (one for each target, sharing a frozen vision encoder). 

As shown in Table \ref{tab:multigoal_pusht}, \ENAP\ outperforms all baselines in both the ``Either" Goal metric (reaching any target, reflecting multi-modal competence) and the strict ``Nearest" Goal metric (pushing toward the target closest to the initial configuration, requiring a higher spatial reasoning capability). This confirms that \ENAP\ avoids collapsing to a single dominant behavior, successfully separating the conflicting data modes into distinct topological paths while retaining the shared pushing skills.

\subsection{Ablation Studies}
We investigate two critical components governing structure discovery: (1) the similarity threshold $\tau_{\text{sim}}$, which determines the granularity of the automaton; and (2) the training of the RNN, which motivates state differentiation. As visualized in \cref{fig:ablation}, performance with an intermediate $\tau_{\text{sim}}$ and all RNN training losses peaks, validating the efficacy of our design. Additional ablations are provided in \cref{sec_app:ic}.

\begin{figure}[t]
    \centering
    \includegraphics[width=\linewidth]{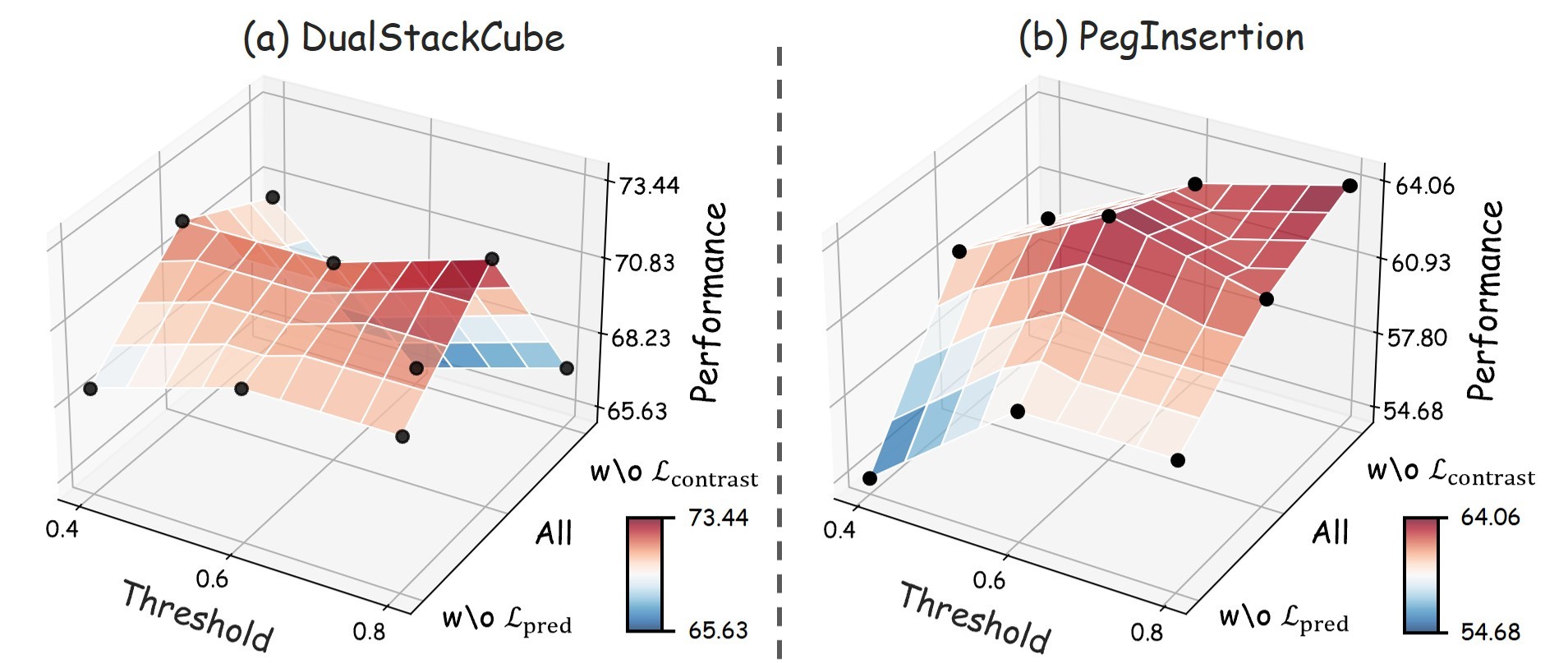}
    \caption{\textbf{Ablation on Structure Discovery Parameters.} Success rates across varying similarity thresholds $\tau_{\text{sim}}$ and RNN loss configurations. Black dots indicate measured data points, while the surface shows interpolated values.}
    \label{fig:ablation}
\end{figure}

\section{Discussion: A POMDP Perspective}
We formulate the justification for \ENAP's bi-level architecture by aligning it with the Bellman equation. Consider the true value function $V(q, o, s)$ of the underlying POMDP, conditioned on the latent task mode $q$, observation $o$, and environment state $s$. The Bellman equation decomposes as:
\begin{equation}
\begin{aligned}
V(q, o, s)
&= \sum_{a, q'} P(q', a \mid q, o)
\Big[
    R(s, a) \\
&\qquad\qquad
    + \gamma \sum_{s', o'} P(s', o' \mid s, a)\, V(q', o', s')
\Big]
\end{aligned}
\end{equation}
where the joint transition-policy distribution $P(q', a | q, o)$ represents the coupled decision process: choosing an action $a$ and transitioning to a new internal mode $q'$. In our framework, this distribution is explicitly factorized using the PMM structure:
\begin{equation}
\begin{aligned}
P(q', a \mid q, o)
&= \underbrace{\delta(q'|q, c)}_{\text{Topological Transition}}
\cdot \sum_{a_{\text{base}}^t}
    \underbrace{\lambda(a_{\text{base}}^t|q, c)}_{\text{Action Prior}} \\
&\qquad\qquad\cdot
    \underbrace{\pi_{\psi}(a \mid q, o, a_{\text{base}}^t)}_{\text{Residual Correction}}
\end{aligned}
\end{equation} 

The PMM captures the global dynamics ($\delta$) and coarse intent ($\lambda$), leaving the neural network to focus solely on the local residual term $\pi_{\psi}$. This reduces the complexity of the function approximation task, explaining the framework's high sample efficiency.

Furthermore, by augmenting the state space to $\tilde{s} = (q, o)$, we enable planning over the belief space. The value function over this augmented state is given by:
\begin{equation}
    \hat{V}(\tilde{s}) = \max_{a} \left[ \hat{r}(\tilde{s}, a) + \gamma \sum_{\tilde{s}'} \hat{P}_{\tilde{s}}(\tilde{s}' | \tilde{s}, a) \hat{V}(\tilde{s}') \right]
\end{equation}
Assuming the learned PMM states $Q$ effectively summarize the history $h$ (i.e., $q$ is a sufficient statistic), we have the approximation $\hat{V}(\tilde{s}) \approx \mathbb{E}_{s \sim P(s|h)} [V(q, o, s)]$. Our explicit recording of next-input signatures on PMM edges serves to approximate the transition dynamics $\hat{P}_{\tilde{s}}$, allowing the agent to deterministically track its state $\tilde{s}$ and execute optimal planning. 

\section{Conclusion}
\noindent\textbf{What is a good abstraction?} Defining a ``good'' abstraction is ill-posed, as the appropriate granularity differs across contexts. In our paper, abstraction quality is performance-driven: the optimal structure is the one that best minimizes the downstream policy objective, which is why different network backbones yield divergent abstractions (\cref{fig:rader}).

\noindent\textbf{Limitations.} First, the quality of the emergent structure is sensitive hyperparameters. Furthermore, scaling to cross-task and cross-embodiment settings remains an open direction, for instance, by extracting abstractions over the semantic token spaces of generalist VLA models.

In this work, we presented \ENAP, a framework for learning emergent neural automaton policies from visuomotor trajectories. By unifying active automata inference with residual control, \ENAP\ unsupervisedly discovers a bi-level structure where a discrete Mealy machine governs high-level planning and networks handle continuous execution. This representation not only enables sample-efficient learning in long-horizon tasks but also provides an explicit, interpretable map of the agent's latent decision logic, thus offering a viable path toward generalizable and interpretable intelligence.

\newpage

\bibliographystyle{plainnat}
\bibliography{references}

\clearpage
\appendices 

\section{$L^*$ Algorithm Walkthrough: Classical vs. Ours}
\subsection{$L^*$ Algorithm}\label{app:L}
To provide intuition for the $L^*$ algorithm, we present a concrete example of learning a target FSM that accepts strings with an even number of $a$'s and $b$'s.

We initialize the observation table with prefixes $\mathcal{U}=\{\epsilon\}$ (candidate states, top portion) and test suffixes $\mathcal{E}=\{\epsilon\}$ (distinguishing experiments, columns); see Tab.~\ref{tab:initial_hyp_a}. The lower portion consists of one-step extensions $\mathcal{U}\cdot\Sigma$ obtained via Membership Queries (MQs). Let $\text{row}(u)$ denote the output vector associated with a prefix $u$ over all tests in $\mathcal{E}$. The table is \textit{closed} if every transition signature in the lower portion is represented in the upper portion. Since the signature of prefix $a$ ($\text{row}(a)=[0]$) differs from that of the initial state ($\text{row}(\epsilon)=[1]$), the table is not closed. We then move $a$ into $\mathcal{U}$ (adding it to the upper portion) to introduce a new state, and automatically generate its one-step extensions ($a \cdot \Sigma$) to populate the lower portion. This results in Tab.~\ref{tab:initial_hyp_b}, where $\cdot$ denotes a one-step extension. The table is now both closed and \textit{consistent} (i.e., equivalent states do not lead to conflicting future transitions), yielding the hypothesis $\mathcal{M}_1$ as in Fig.\ref{fig:phase1}.

\begin{figure}[h]
    \centering
    \begin{minipage}[c]{0.3\linewidth}
        \centering
        \captionof{table}{Initial \\ Hypothesis}\label{tab:initial_hyp_a}
        \begin{tabular}{r|c}
        \hline
        $T_1$ & $\epsilon$ \\ \hline
        $\epsilon$ & 1 \\ \hline
        $a$ & 0 \\ 
        $b$ & 0 \\  \hline
        \end{tabular}
    \end{minipage}
      \begin{minipage}[c]{0.3\linewidth}
        \centering
        \captionof{table}{Updated \\ Hypothesis}\label{tab:initial_hyp_b}
        \begin{tabular}{r|c}
        \hline
        $T_1$ & $\epsilon$ \\ \hline
        $\epsilon$ & 1 \\
        $a$ & 0 \\ \hline
        $b$ & 0 \\ 
        $a \cdot a$ & 1 \\ 
        $a \cdot b$ & 0 \\ \hline
        \end{tabular}
    \end{minipage}
    \begin{minipage}[c]{0.35\linewidth}
        \centering
        \includegraphics[width=0.75\linewidth]{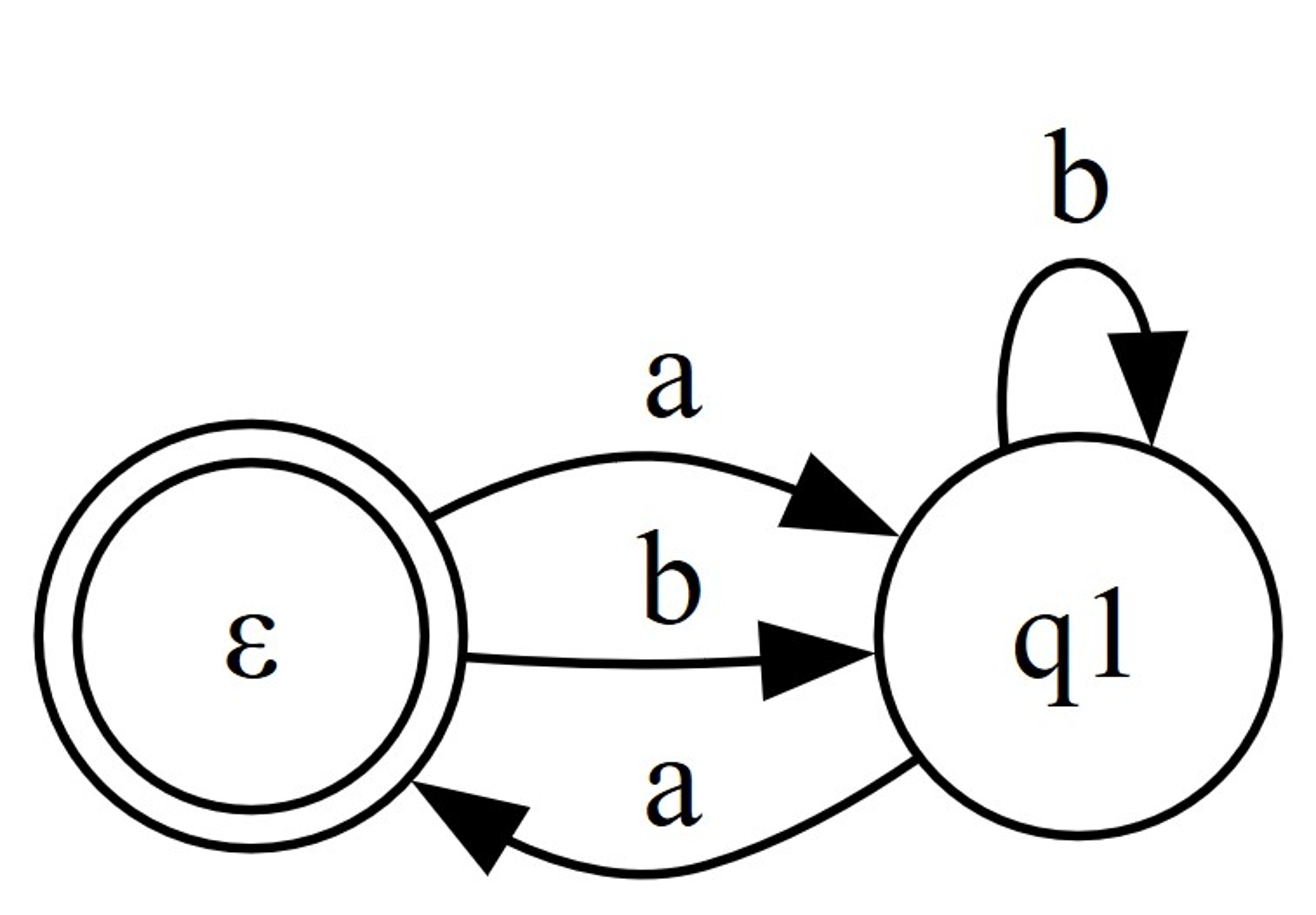}
        \vspace{2pt}
        \caption{\textbf{Hypothesis} $\mathcal{M}_1$}
        \label{fig:phase1}
    \end{minipage}
    \vspace{-5pt}
\end{figure}

However, $\mathcal{M}_1$ fails to accept the counterexample $bb$. To address this, we incorporate $bb$ along with all its prefixes (in this case, only $b$) into the state set $\mathcal{U}$, forming the upper portion of Tab.~\ref{tab:phase2}. As a result, the lower portion is automatically extended by issuing Membership Queries (MQs) for all one-step extensions ($u \cdot c$ for $u \in \mathcal{U}, c \in \Sigma$). 


Initially, this addition makes the table \textit{inconsistent}. Specifically, under the single test $\epsilon$, states $a$ and $b$ appear indistinguishable ($\text{row}(a)=\text{row}(b)=[0]$). However, extending them with symbol $a$ exposes a behavioral discrepancy, since $\text{row}(a \cdot a)=[1] \neq \text{row}(b \cdot a)=[0]$. To resolve this structural ambiguity, we must introduce the distinguishing suffix $a$ as an additional test column in $\mathcal{E}$. The resulting table, shown as Tab.~\ref{tab:phase2}, successfully separates the previously confounded states $a$ and $b$, and yields the updated hypothesis $\mathcal{M}_2$.


\begin{figure}[h]
    \centering
    \begin{minipage}[c]{0.49\linewidth}
        \centering
         \captionof{table}{Closed but Inconsistent}\label{tab:phase2}
        \begin{tabular}{r|cc}
        \hline
        $T_2$ & $\epsilon$ & $a$ \\ \hline
        $\epsilon$ & 1 & 0 \\
        $a$ & 0 & 1 \\
        $b$ & 0 & 0 \\ 
        $bb$ & 1 & 0 \\ \hline
        $aa$ & 1 & 0 \\
        $ab$ & 0 & 0 \\
        $b  \cdot a$ & 0 & 0 \\
        $bb  \cdot a$ & 0 & 1 \\
        $bb \cdot b$ & 0 & 0 \\ \hline
        \end{tabular}
    \end{minipage}
    \hfill
    \begin{minipage}[c]{0.49\linewidth}
        \centering
        \includegraphics[width=0.9\linewidth]{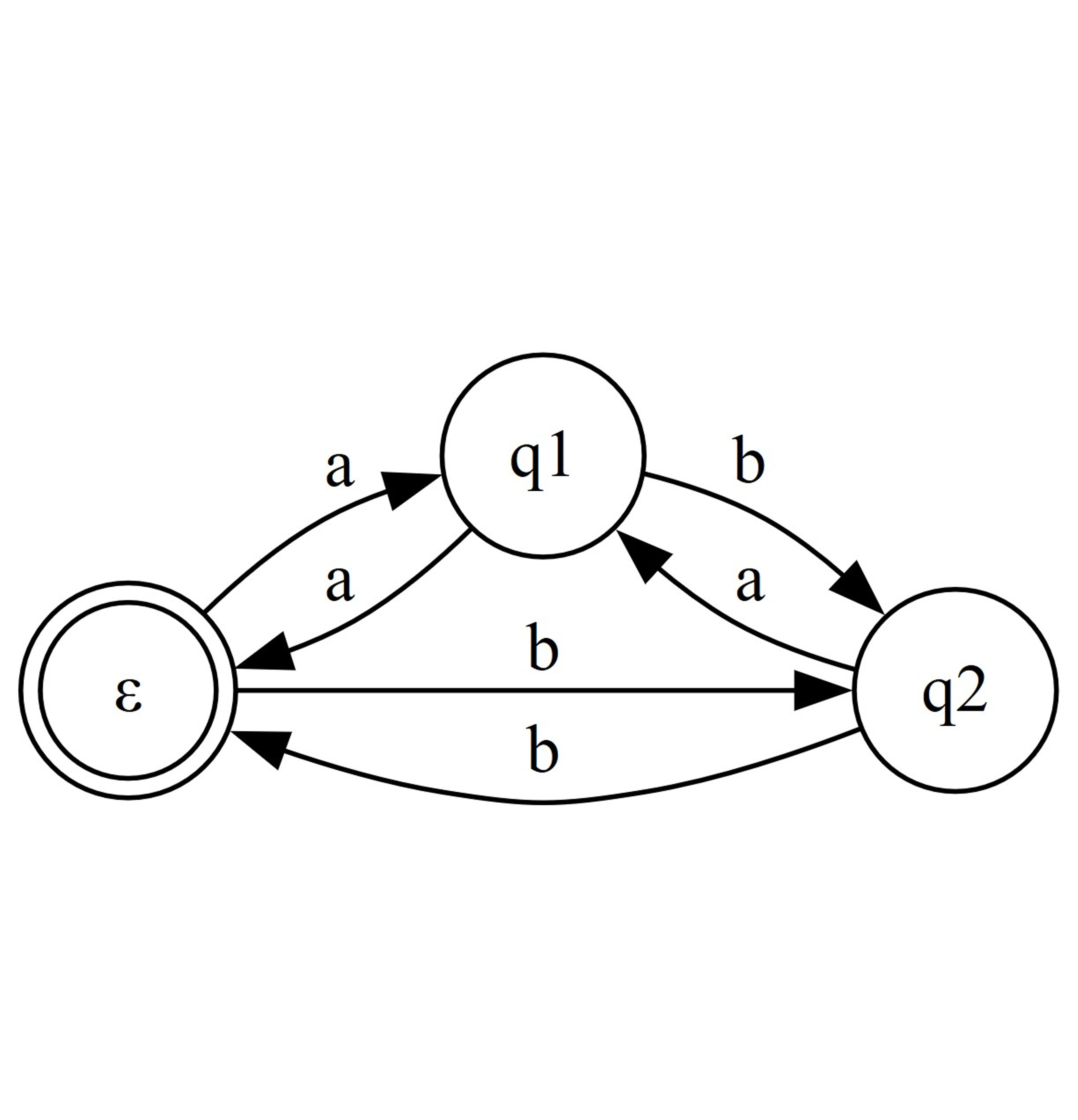}
        \caption{\textbf{Hypothesis} $\mathcal{M}_2$}
        \label{fig:phase2}
    \end{minipage}
    \vspace{-5pt}
\end{figure}


Although Tab.~\ref{tab:phase2} resolves the immediate inconsistency, subsequent equivalence queries and closure checks reveal further structural deficits (e.g., distinguishing the states reached via $b$ and $ab$). The $L^*$ algorithm naturally iterates through alternating phases of state promotion (to restore closedness) and suffix addition (to restore consistency). After adding the suffix $b$ and promoting the necessary intermediate prefixes (such as $ab$), the algorithm converges to Tab.~\ref{tab:final_phase}. This final table is strictly closed and consistent, successfully identifying the four distinct states isomorphic to the target language.

\begin{figure}[t]
    \centering
    \begin{minipage}[c]{0.39\linewidth}
        \centering
        \captionof{table}{Final Hypothesis}\label{tab:final_phase}
        \resizebox{\linewidth}{!}{
        \begin{tabular}{r|ccc}
        \hline
        $T_3$ & $\epsilon$ & $a$ & $b$ \\ \hline
        $\epsilon$ & 1 & 0 & 0 \\
        $a$        & 0 & 1 & 0 \\
        $b$        & 0 & 0 & 1 \\
        $bb$ & 1 & 0 & 0 \\
        $ab$ & 0 & 0 & 0 \\
        $abb$ & 0 & 1 & 0 \\ \hline
        $aa$ & 1 & 0 & 0 \\
        $ba$ & 0 & 0 & 0 \\
        $bba$ & 0 & 1 & 0 \\
        $bbb$ & 0 & 0 & 1 \\
        $ab \cdot a$ & 0 & 0 & 1 \\
        $abb \cdot a$ & 1 & 0 & 0 \\
        {$abb \cdot b$} & 0 & 0 & 0 \\ \hline
        \end{tabular}
        }
    \end{minipage}
    \hfill
    \begin{minipage}[c]{0.49\linewidth}
        \centering
        \includegraphics[width=\linewidth]{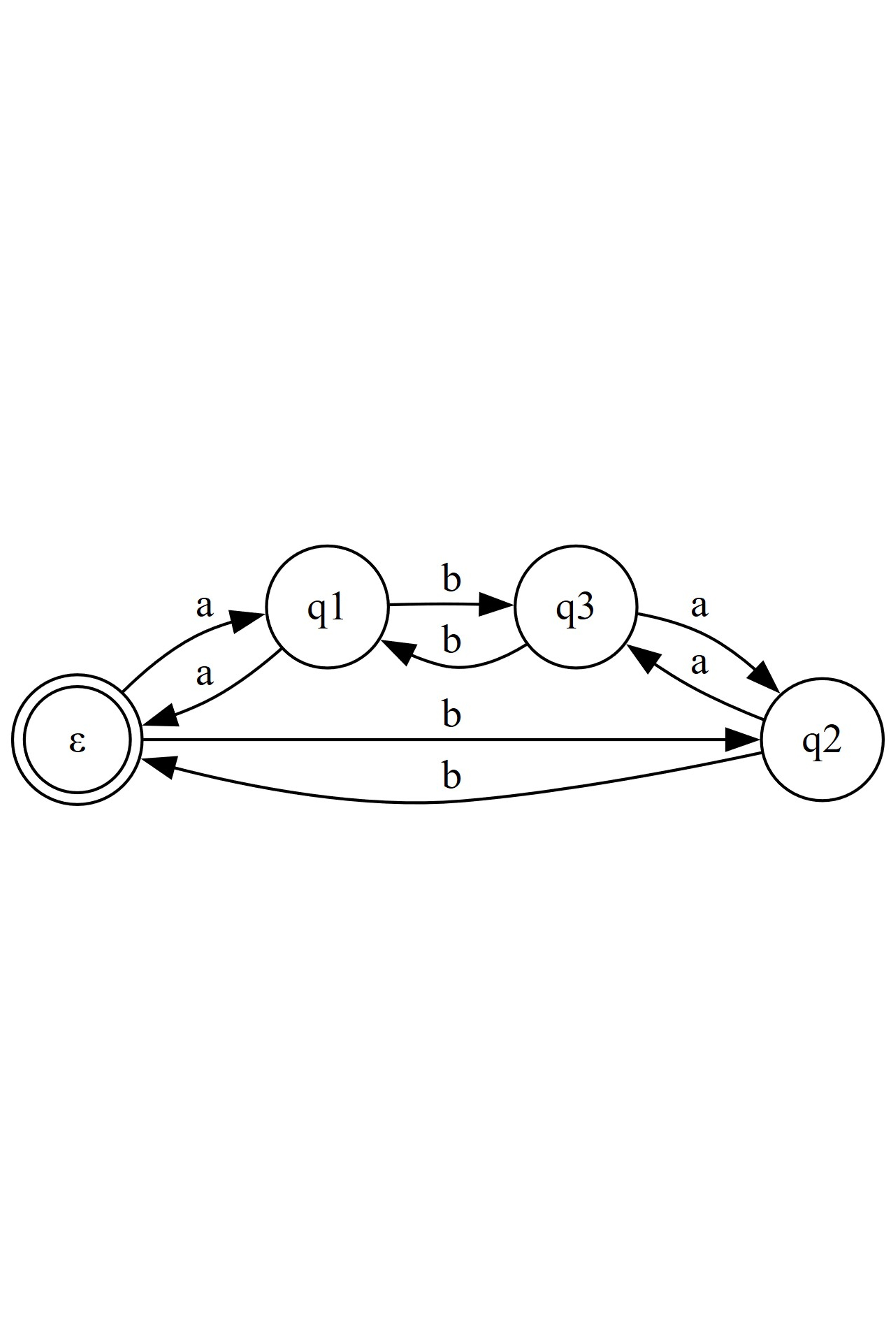}
        \caption{\textbf{Final} $\mathcal{M}_3$}
        \label{fig:final}
    \end{minipage}
    \vspace{-15pt}
\end{figure}

\subsection{Extended $L^*$ Algorithm (\cref{alg:harp_discovery})}
\subsubsection{\textbf{Experiment Setup}}
To provide a transparent walkthrough of the structure extraction process, we utilize the discrete $4 \times 4$ FrozenLake environment, where the agent can choose actions $\{\rm Up(\textit{U}), Down(\textit{D}), Right(\textit{R}), Left(\textit{L})\}$ to reach the goal; see Fig.~\ref{fig:frozenlake}. We use RNNs with random initialization. Since the state space is inherently discrete, we bypass the clustering phase and map the grid position directly to the input alphabet $\Sigma$, such that the symbol $c_t$ corresponds to the grid index (e.g., $c_{14}$ represents position 14). We collect two successful trajectories from Start ($c_0$) to Goal ($c_{15}$) that share a common history but diverge at a decision point $c_9$. 
\begin{itemize}
    \item $\tau_1$: $c_0 \xrightarrow{D} c_4 \xrightarrow{D} c_8 \xrightarrow{R} c_9 \xrightarrow{D} \mathbf{c_{13}} \xrightarrow{R} c_{14} \xrightarrow{R} c_{15}$
    \item $\tau_2$: $c_0 \xrightarrow{D} c_4 \xrightarrow{D} c_8 \xrightarrow{R} c_9 \xrightarrow{R} \mathbf{c_{10}} \xrightarrow{D} c_{14} \xrightarrow{R} c_{15}$
\end{itemize}

\subsubsection{\textbf{Iteration}}
This setup is specifically chosen to demonstrate how the algorithm iteratively discovers the topological branching at state $c_9$ and the subsequent state convergence at $c_{14}$. In the following figures, edges are annotated with the tuple: input symbol, action list, probability, and ${\rm NIS}(\cdot)$.

The algorithm initializes with the start state $q_0$ (history $\epsilon$) and expands via \textit{Self-Supervised MQ} to find reachable states. It identifies $q_1$ (history $0$) via the transition enabled by $c_0$; see~\cref{fig:iter1}. However, due to limited exploration, the initial hypothesis  only loops at $q_1$.
Running the \textit{NFA-style EQ} on trajectory $\tau_1$ reveals a failure: the sequence $c_0 \xrightarrow{D} c_4 \xrightarrow{D} c_8$ cannot be explained. The prefix $(c_0, c_4, c_8)$ is returned as a counterexample.

\begin{wrapfigure}{r}{0.2\textwidth}
\vspace{-11pt}
\centering
\includegraphics[width=\linewidth]{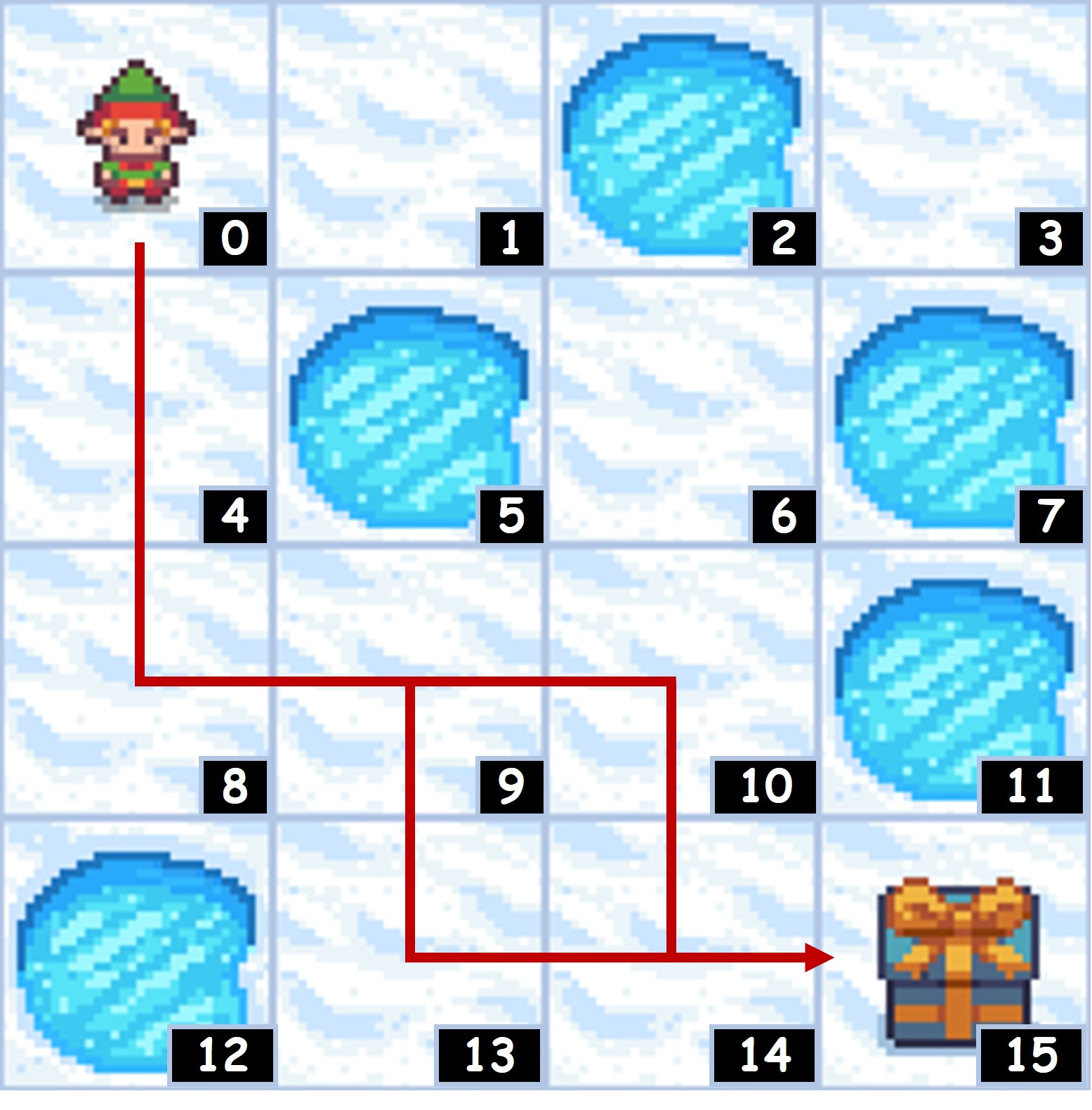}
\vspace{-15pt}
\caption{\textbf{FrozenLake Map}.}
\label{fig:frozenlake}
\vspace{-5pt}
\end{wrapfigure}

\begin{figure*}[t]
    \centering
    \begin{minipage}[c]{0.25\linewidth}
        \centering
        \includegraphics[width=\linewidth]{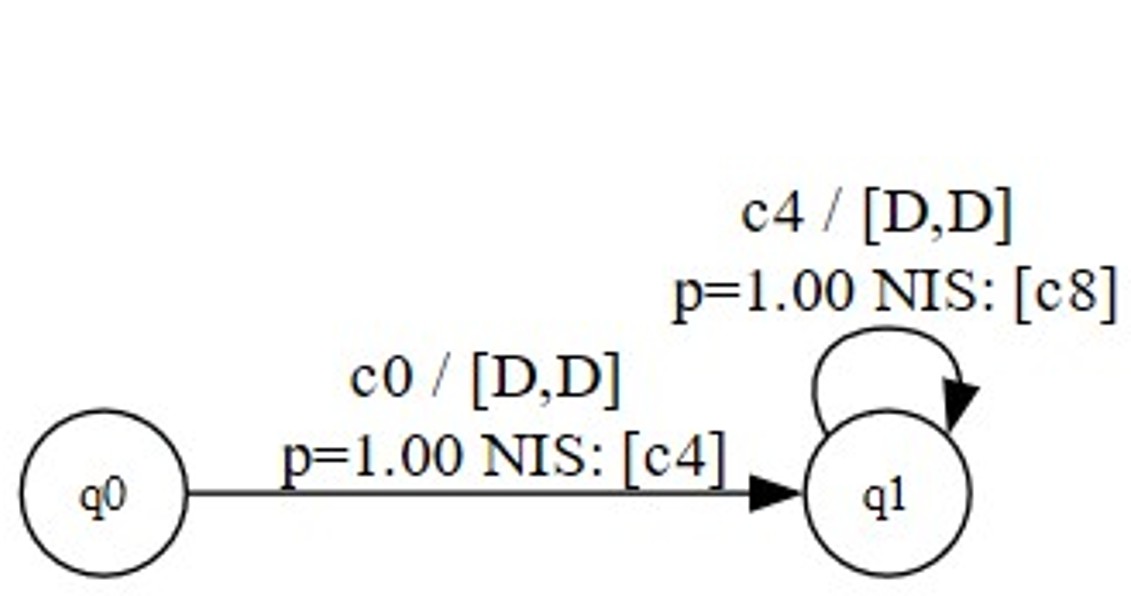}
        \caption{\textbf{Iteration 1}: Initial hypothesis.}
        \label{fig:iter1}
    \end{minipage}
    \hfill
    \begin{minipage}[c]{0.63\linewidth}
        \centering
        \includegraphics[width=\linewidth]{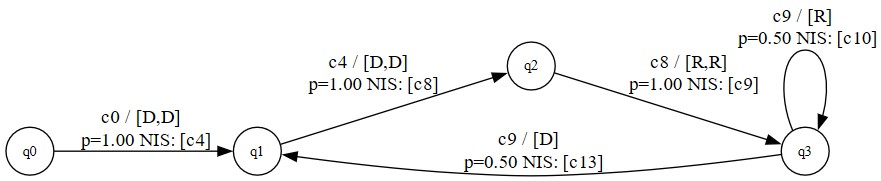}
        \caption{\textbf{Iteration 2}: Linear chain extension.}
        \label{fig:iter2}
    \end{minipage}
\end{figure*}

\begin{figure*}[htbp]
    \centering
    \includegraphics[width=0.9\linewidth]{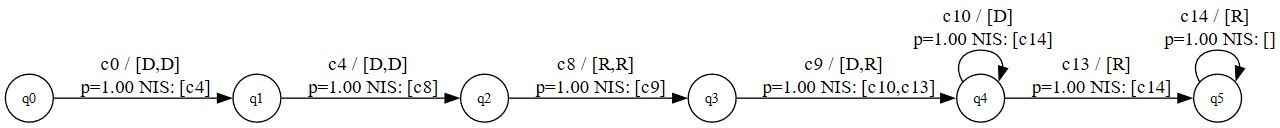} 
    \vspace{-5pt}
    \caption{\textbf{Iteration 3}: Discovery of the decision point.}
    \label{fig:iter3}
\end{figure*}

\begin{figure*}[h]
    \centering
    \includegraphics[width=\linewidth]{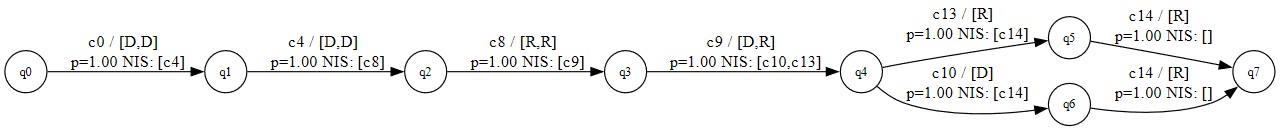} 
    \vspace{-15pt}
    \caption{\textbf{Iteration 4}: Stabilized hypothesis capturing the full topology.}
    \label{fig:iter4}
\end{figure*}

Adding the counterexample to the prefix set $\mathcal{U}$ triggers \textit{Closedness} checks, expanding the state space to cover the history $c_0 \to c_4 \to c_8$. The new hypothesis (\cref{fig:iter2}) successfully models the linear path up to $c_9$. However, it fails to validly explain the divergence into $c_{13}$ or $c_{10}$. The prefix $(c_0, c_4, c_8, c_9, c_{13})$ is returned as a counterexample.

Incorporating the new counterexample, the algorithm discovers the critical branching state $q_4$ in \cref{fig:iter3}. The MQ reveals two distinct transitions from $q_4$: one leading to $q_5$ and another leading to $q_4$ itself via self-loop $c_{10}$. The hypothesis  now captures the ``lower" path ($\tau_1$) correctly but still misinterprets the ``upper" path ($\tau_2$) convergence, failing at $(c_0, c_4, c_8, c_9, c_{10}, c_{14})$.

The final counterexample forces the expansion of the ``lower" path state $q_6$ (history $c_0 \to \cdots \to c_{10}$). Crucially, the MQ discovers that both $q_5$ (from $\tau_1$) and $q_6$ (from $\tau_2$) transition to the same next-step region under input $c_{14}$, which maps to a unified state $q_7$.
The resulting hypothesis (\textit{Stable Phase Pruning} applied; see \cref{fig:iter4}) correctly models the full topology: common past $\to$ branching $\to$ convergence. The EQ finds no further counterexamples, and the hypothesis stabilizes.


\section{Theoretical Motivation for RNN Encoder}
\label{app:rnn_theory}

Our framework relies on the premise that the continuous hidden states of a trained RNN can represent the discrete states of an FSM. This architectural choice is theoretically motivated by the \textit{saturation} (the hidden states $h_t \in \mathbb{R}^h$ tend to cluster around the vertices of the hypercube $\{\pm 1\}^d$ as training progresses) dynamics of RNNs with $\tanh$ activations. This allows us to treat the continuous vector space as an approximation of a discrete state space. The following proposition formalizes the conditions under which cosine similarity serves as a reliable metric for distinguishing these discrete states.

\vspace{7pt}
\begin{proposition}[Identifiability via $\epsilon$-Saturated RNNs]
\label{prop:saturation_proof}
Let $h_1, h_2 \in \mathbb{R}^h$ be two normalized hidden state vectors. Let $\tilde{h}_1, \tilde{h}_2 \in \{\pm 1\}^d$ be their corresponding saturated discrete versions. Assume the RNN operates in an $\epsilon$-saturated regime, such that $\|h_i - \tilde{h}_i\|_2 \leq \epsilon$ for $i \in \{1, 2\}$. 
Then, if the cosine similarity satisfies:
\begin{equation*}
    \cos(h_1, h_2) \geq 1 - \kappa \quad \text{with} \quad \sqrt{\kappa} < \sqrt{2} \left( \frac{1}{\sqrt{d}} - \epsilon \right),
\end{equation*}
The two vectors correspond to the same discrete automaton state (i.e., $\tilde{h}_1 = \tilde{h}_2$).
\end{proposition}
\vspace{7pt}

\begin{proof}
By the triangle inequality, the distance between the discrete states is bounded by:
\begin{equation*}
\begin{aligned}
\|\tilde{h}_1 - \tilde{h}_2\|_2 &= \|\tilde{h}_1 - h_1 + h_1 - h_2 + h_2 - \tilde{h}_2\|_2 \\
&\leq \|\tilde{h}_1 - h_1\|_2 + \|h_1 - h_2\|_2 + \|h_2 - \tilde{h}_2\|_2 \\
&\leq 2\epsilon + \|h_1 - h_2\|_2.
\end{aligned}
\end{equation*}
Since saturated RNN states take discrete values in $\{\pm 1\}^d$, the minimum non-zero Euclidean distance between any two distinct vertices is $2$. Then, using the relationship between Euclidean distance and cosine similarity for normalized vectors: $\|h_1 - h_2\|_2 = \sqrt{2(1 - \cos(h_1, h_2))}$.
Given the condition $\cos(h_1, h_2) \geq 1 - \kappa$, we have $\|h_1 - h_2\|_2 \leq \sqrt{2\kappa}$.
Substituting this into the inequality:
\begin{equation*}
\|\tilde{h}_1 - \tilde{h}_2\|_2 \leq 2\epsilon + \sqrt{2\kappa}.
\end{equation*}
For $\tilde{h}_1$ and $\tilde{h}_2$ to be the same state, the error bound must be strictly less than the minimum distance between distinct discrete states. The condition $\sqrt{\kappa} < \sqrt{2} (1/\sqrt{d} - \epsilon)$ ensures that the vectors fall within the same decision boundary of the discrete state, implying $\tilde{h}_1 = \tilde{h}_2$.
\newline
\end{proof}

This result confirms that by enforcing saturation and using a strict cosine threshold, we can rigorously recover the underlying discrete phases from the continuous RNN embeddings.

\section{Experiments: Implementation Details}\label{app:experiments}
\subsection{Task Specifications}
We evaluate \ENAP\ across three domains, each with distinct robotic platforms and observation spaces. For \textit{Complex Manipulation}, we utilize the ManiSkill benchmark (Franka Emika Panda arm), providing the agent with RGB images and proprioceptive states. For \textit{Long-Horizon TAMP}, we adopt the CALVIN environment (Franka Panda arm), where observations include base and wrist camera feeds alongside language task descriptions. Finally, our \textit{Real-World Manipulation} experiments employ a Kinova Gen3 arm, accessing base/wrist camera streams and end-effector pose. We collected $\sim$400 and $\sim$25 trajectories for each task in the simulated and real-world experiments, respectively. Specific task definitions and success criteria are detailed below.

\subsubsection{\textbf{PegInsertionSide (Sim)}} 
This task requires picking up a bi-colored peg and inserting either the orange or white end into a horizontal box hole. The environment features significant variability: peg lengths are sampled from $U[0.15, 0.20]\text{m}$ and radii from $U[0.015, 0.025]\text{m}$, with initial poses randomized in position and rotation. Success is defined as inserting either end at least $0.03\text{m}$ into the hole (radius $r_{\text{peg}}+0.02\text{m}$).

\subsubsection{\textbf{DualStackCube (Sim)}} 
The agent must stack two colored cubes (Red/Green) in any valid vertical configuration. Cube positions are initialized uniformly within a $0.2\text{m} \times 0.4\text{m}$ region. A trial is considered successful only if the top cube aligns within half-width of the bottom cube, remains static, and is fully released by the gripper.

\subsubsection{\textbf{MultiGoalPushT (Sim)}} 
A multi-modal variant of the Push-T task where the agent must maneuver a T-shaped block to one of two potential target poses, with success defined either by reaching any valid target or the nearest one depending on the evaluation setting. The T-block's initial pose is fully randomized. Success is achieved when the intersection-over-union (IoU) between the block and the designated target silhouette exceeds $0.90$.

\subsubsection{\textbf{Sequential (Sim)}} 
The task enforces a strict causal chain: (1) rotate the red block right ($>60^\circ$ yaw); (2) push the blue block left ($x$-displacement $<-0.1$m); (3) move the slider right (joint $\Delta < -0.15$); (4) toggle the LED switch; and (5) turn on the light bulb.

\subsubsection{\textbf{Hierarchical (Sim)}} 
The task combines articulated object manipulation and imposes temporal logic (e.g., a drawer must be opened before placing the block inside it): (1) open the drawer (joint $\Delta \ge 0.12$); (2) lift the red block from the table ($z > 0.05$m); (3) place it inside the drawer; (4) lift the pink block from the slider ($z > 0.03$m); and (5) stack it onto another object (static stability required).

\subsubsection{\textbf{Hanger (Real)}} 
The task requires unhooking a clothes hanger from one end of a rack and transferring it to the opposite side. An obstacle in the middle prevents direct lateral movement, forcing a complex lift-over trajectory. Success is defined as the hanger being stably hooked on the target side.

\subsubsection{\textbf{MultiPickPlace (Real)}} 
It involves sorting three colored cans (red, green, blue) into corresponding bowls (red, green, blue) fixed on the table. The cans are initialized randomly on the same side, requiring the agent to identify and match object-container pairs.

\subsubsection{\textbf{StackLego (Real)}} 
The task demands high-precision assembly without force feedback. A blue Lego brick is initialized at discrete locations on a green base plate and must be stacked onto a fixed red brick. We adopt a graded success metric: 1.0 for a full insertion (studs interlocked) and 0.5 for a stable stack (brick resting on top) without interlocking.
\begin{figure*}[t]
    \centering
    \includegraphics[width=\linewidth]{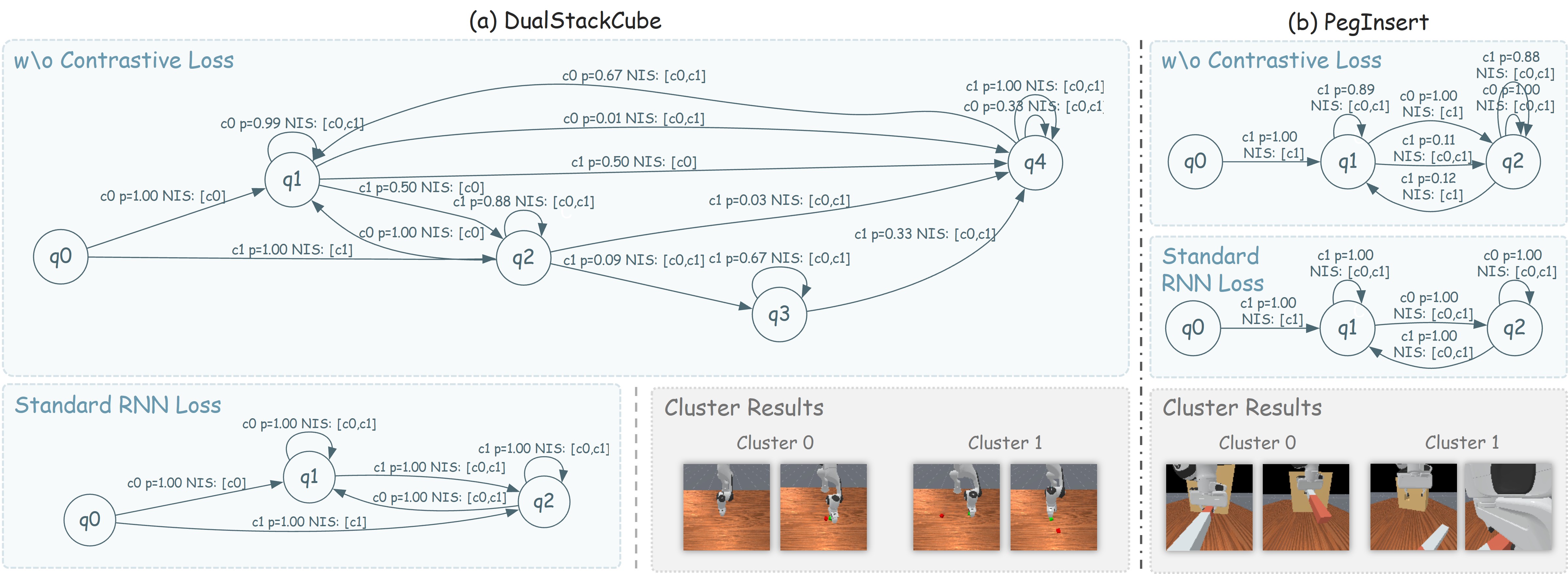}
    \vspace{-15pt}
    \caption{\textbf{Impact of Phase-Aware Contrastive Loss.} Without phase-aware contrastive loss, the automaton is fragmented with ambiguous transitions.}
    \label{fig:app_analysis}
\end{figure*}
\begin{figure}[t]
    \centering
    \includegraphics[width=\linewidth]{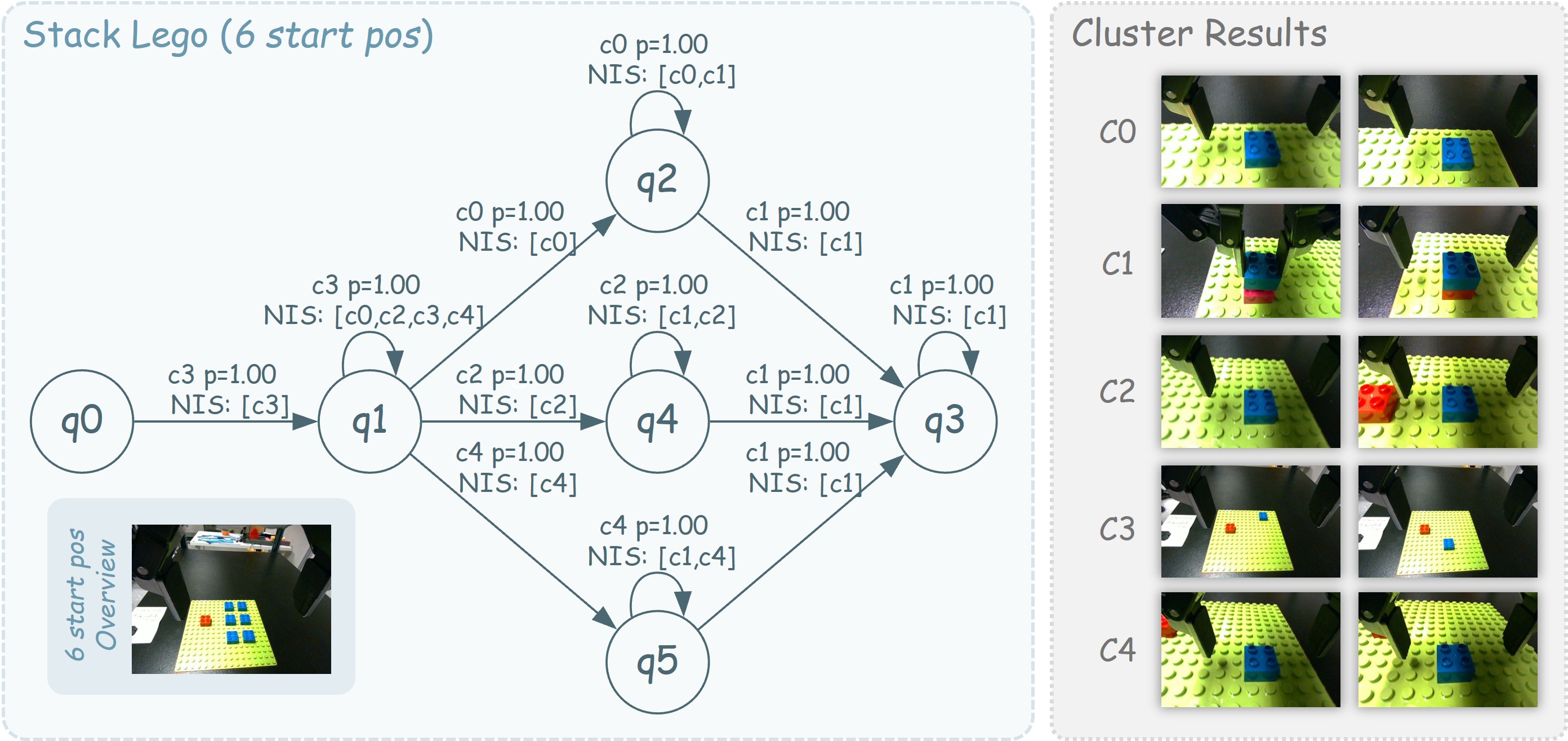}
    \caption{\textbf{Generalization of Structural Branching.} Instead of overfitting to six distinct branches, the model discovers three logical pathways ($q_2, q_4, q_5$), grouping visually and kinetically similar poses (e.g., along the line of sight) into shared phases.}
    \vspace{-5pt}
    \label{fig:6pos}
\end{figure}

\begin{table}[h]
\centering
\caption{{Network Architectures for Complex Manipulation.}}
\vspace{-5pt}
\label{tab:complex_arch}
\resizebox{\linewidth}{!}{%
\begin{tabular}{lcc}
\hline
{Component}\rule{0pt}{7pt} & \textbf{\ENAP\ (Oracle)} & {\textcolor{harpcolor}{\textbf{ENAP$^*$}} (DINO)} \\
\hline
\multicolumn{3}{l}{\cellcolor{gray!10}\textit{{- Encoder Architecture}}\rule{0pt}{7pt}} \\
Visual Backbone & NatureCNN (3-layer Conv) & DINOv2 ViT-S/14 (Frozen) \\
Visual Head & Linear $\to$ 256 & Spatial Softmax $\to$ MLP $\to$ 256 \\
Proprioception & Linear $\to$ 256 & Fourier (7-freq) $\to$ SpecNorm $\to$ 128 \\
\multicolumn{3}{l}{\cellcolor{gray!10}\textit{{- Feature Dimensions (Vis + Pos [+ Delta\_Pos])}}\rule{0pt}{7pt}} \\
\texttt{StackCube} & 512 (256+256) & 384 (256+128) \\
\texttt{Peg} / \texttt{PushT} & 512 (256+256+0) & 512 (256+128+128) \\
\multicolumn{3}{l}{\cellcolor{gray!10}\textit{{- Residual Policy Network (Hidden Units)}}\rule{0pt}{7pt}} \\
\texttt{StackCube} & 512 & 512 \\
\texttt{Peg} / \texttt{PushT} & 256 & 256 \\
\multicolumn{3}{l}{\cellcolor{gray!10}\textit{{- PMM Structure Discovery}}\rule{0pt}{7pt}} \\
RNN Encoder & \multicolumn{2}{c}{Hidden Dim 64} \\
State Embed & \multicolumn{2}{c}{Learnable Embedding (Dim 16)} \\
\hline
\end{tabular}%
}
\end{table}
\begin{table}[h]
\centering
\caption{{Network Architectures for Long-Horizon TAMP.}}
\vspace{-5pt}
\label{tab:tamp_arch}
\resizebox{\linewidth}{!}{%
\begin{tabular}{lc}
\hline
\textbf{Component}\rule{0pt}{7pt} & {\ENAP\ (FLOWER)} \\
\hline
\multicolumn{2}{l}{\cellcolor{gray!10}\textit{{- Encoder Architecture}}\rule{0pt}{7pt}} \\
Visual Backbone & Florence-2 DaVit (Frozen) \\
Language Encoder & Florence-2 Text Encoder \\
Fusion & Concat $\to$ Florence-2 Encoder $\to$ Mean Pool \\
Feature Dim & 1024 \\
\multicolumn{2}{l}{\cellcolor{gray!10}\textit{{- Residual Policy Network}}\rule{0pt}{7pt}} \\
Architecture & Linear (2118 $\to$ 256) $\to$ ReLU $\to$ Linear (256 $\to$ 70) \\
Hidden Units & 256 \\
Output & 70-dim Action Sequence ($10 \times 7$) \\
\multicolumn{2}{l}{\cellcolor{gray!10}\textit{{- PMM Structure Discovery}}\rule{0pt}{7pt}} \\
RNN Encoder & Hidden Dim 64 \\
State Embed & Learnable Embedding (Dim 16) \\
Input Dim & Action (70) + State Embed (16) \\
\hline
\end{tabular}%
}
\end{table}
\begin{table}[h]
\centering
\caption{{Network Architectures for Real-World Manipulation.}}
\vspace{-5pt}
\label{tab:real_arch}
\resizebox{\linewidth}{!}{%
\begin{tabular}{lc}
\hline
\textbf{Component}\rule{0pt}{7pt} & {\textcolor{harpcolor}{\textbf{ENAP$^*$}} (DINO)} \\
\hline
\multicolumn{2}{l}{\cellcolor{gray!10}\textit{{- Encoder Architecture}}\rule{0pt}{7pt}} \\
Visual Backbone & DINOv2 ViT-S/14 (Frozen) \\
Visual Head & Spatial Softmax $\to$ MLP (768$\to$128) \\
Proprioception & Fourier (7-freq) $\to$ SpecNorm $\to$ 128 \\
Fusion & Base (128) + Wrist (128) + Prop (128) = 384 \\
\multicolumn{2}{l}{\cellcolor{gray!10}\textit{{- Policy Network}}\rule{0pt}{7pt}} \\
Residual MLP & Linear (771 $\to$ 256) $\to$ ReLU $\to$ Output (3) \\
Gripper Head & MLP (Feat + Center $\to$ 128 $\to$ 2) \\
Controller & Cartesian PID ($K_p$=5.0, $K_d$=0.4) \\
\multicolumn{2}{l}{\cellcolor{gray!10}\textit{{- PMM Structure Discovery}}\rule{0pt}{7pt}} \\
RNN Encoder & Hidden Dim 64 \\
State Embed & Learnable Embedding \\
\hline
\end{tabular}%
}
\end{table}
\vspace{-10pt}

\subsection{Network Architecture}
To ensure reproducibility, we detail the specific architectures for the encoders, RNNs, and residual networks, along with the hyperparameters used for training and structure discovery. We provide detailed network specifications for the complex manipulation tasks, the long-horizon TAMP tasks, and the real-world task in \cref{tab:complex_arch}, \cref{tab:tamp_arch} and \cref{tab:real_arch}, respectively. 

\section{Experiments: Further Analysis}
\subsection{Impact of Contrastive Loss on PMM}
\label{sec_app:ic}
To further validate our design, we visualize the impact of the phase-aware contrastive loss $\mathcal{L}_{\text{contrast}}$ on the extracted PMM structure in \cref{fig:app_analysis}. We apply K-means clustering with the number of effective clusters inferred by HDBSCAN to filter out noise points. When training the RNN without this loss (Top), the resulting automaton is fragmented and highly stochastic, with numerous spurious transitions. This indicates that the latent states fail to effectively summarize history. This confirms that explicitly shaping the latent geometry is crucial for discovering interpretable and control-relevant structures.

\subsection{Adaptive Abstraction in Complex Multi-Modal Tasks}
\label{app-sec:aa}
We further investigate the scaling of structure discovery by extending the StackLego task to six initial positions. We assign points from the noise cluster obtained by HDBSCAN to the nearest valid cluster. As shown in \cref{fig:6pos}, the extracted PMM does not naively spawn six separate branches; instead, it adaptively groups kinetically similar poses (e.g., those aligned along the camera's optical axis) into three shared latent modes. This demonstrates that \ENAP\ is not merely memorizing task instances but is optimizing for a compact policy representation that minimizes the residual learning burden, effectively discovering the ``minimal sufficient structure" for control.

\end{document}